\pdfoutput=1
\documentclass[10pt,twocolumn,letterpaper]{article}

\def\paperID{2618}
\def\confName{ICCV}
\def\confYear{2025}

\def\authorBlock{
    Caner Korkmaz\textsuperscript{1} \qquad
    Brighton Nuwagira\textsuperscript{2} \qquad
    Barış Coşkunuzer\textsuperscript{2} \qquad
    Tolga Birdal\textsuperscript{1} \\
    \\
    \textsuperscript{1} Imperial College London \qquad
    \textsuperscript{2} UT Dallas
}

\newif\ifreview 
\newif\ifarxiv \newcommand{\arxiv}{\arxivtrue}
\newif\ifcamera 
\newif\ifrebuttal 

\arxiv

\ifreview \usepackage[review]{cvpr} \fi
\ifarxiv \usepackage[pagenumbers]{cvpr} \fi
\ifrebuttal \usepackage[rebuttal]{cvpr} \fi
\ifcamera \usepackage{cvpr} \fi

\usepackage{graphicx}	
\usepackage{amsmath}	
\usepackage{amssymb}	
\usepackage{booktabs}
\usepackage{times}
\usepackage{microtype}
\usepackage{epsfig}
\usepackage{caption}
\usepackage{float}
\usepackage{color, colortbl}
\usepackage{stfloats}
\usepackage{enumitem}
\usepackage{tabularx}
\usepackage{xstring}
\usepackage{multirow}
\usepackage{xspace}
\usepackage{url}
\usepackage{subcaption}
\usepackage{xcolor}

\ifcamera \usepackage[accsupp]{axessibility} \fi

\usepackage{amsfonts}       %
\usepackage{nicefrac}       %
\usepackage{microtype}      %
\usepackage{xcolor}         %
\usepackage{graphicx}
\usepackage{amsmath,mathrsfs}
\usepackage{amssymb}
\usepackage{mathtools,pifont}
\usepackage{amsthm}
\usepackage{wrapfig,lipsum}

\usepackage[linesnumbered,ruled,vlined]{algorithm2e}

\newcommand{\R}{\mathbb{R}}

\newcommand{\Z}{\mathbb{Z}}

\newcommand{\PD}{\mathrm{PD}}

\newcommand{\wh}{\widehat}

\newcommand{\N}{\mathcal{N}}

\newcommand{\I}{\mathcal{I}}
\newcommand{\D}{\mathcal{D}}
\newcommand{\W}{\mathcal{W}}
\newcommand{\B}{\mathbf{B}}

\newcommand{\G}{\mathcal{G}}
\newcommand{\M}{\mathcal{M}}

\newcommand{\X}{\mathcal{X}}

\newcommand{\CC}{\mathcal{K}}
\newcommand{\CCfilt}{\wh{\CC}}
\newcommand{\VectSP}{\Psi}
\newcommand{\VectMP}{\Psi_{\mathrm{MP}}}
\newcommand{\Vect}{\varphi}

\newtheorem{theorem}{Theorem}

\newtheorem*{theorem*}{Theorem}

\newtheorem{thm}{Theorem}
\newtheorem{remark}{Remark}
\newtheorem{cor}{Corollary}
\newtheorem{example}{Example}
\newtheorem{prop}{Proposition}
\newtheorem{dfn}{Definition}

\ifarxiv  \fi

\usepackage{xr-hyper}

\makeatletter
\newcommand*{\addFileDependency}[1]{
  \typeout{(#1)}
  \@addtofilelist{#1}
  \IfFileExists{#1}{}{\typeout{No file #1.}}
}

\makeatother
\newcommand*{\myexternaldocument}[1]{
    \externaldocument{#1}
    \addFileDependency{#1.tex}
    \addFileDependency{#1.aux}
}

\definecolor{cvprblue}{rgb}{0.21,0.49,0.74}
\usepackage[pagebackref,breaklinks,colorlinks,allcolors=cvprblue]{hyperref}
\usepackage[capitalize]{cleveref}
\crefname{equation}{\text{Eq}}{\text{Eq}}
\crefname{definition}{\text{Dfn.}}{\text{Dfn.}}
\crefname{lemma}{\text{Lemma}}{\text{Lemma}}
\crefname{dfn}{\text{Dfn.}}{\text{Dfn.}}
\crefname{thm}{\text{Thm.}}{\text{Thm.}}
\crefname{tab}{\text{Tab.}}{\text{Tab.}}
\crefname{fig}{\text{Fig.}}{\text{Fig.}}
\crefname{table}{\text{Tab.}}{\text{Tab.}}
\crefname{figure}{\text{Fig.}}{\text{Fig.}}
\crefname{section}{\text{Sec.}}{\text{Sec.}}

\crefname{table}{Table}{Tables}
\crefname{figure}{Fig.}{Figs.}

\ifarxiv \crefname{appendix}{App.}{Apps.}
\else \crefname{appendix}{Suppl.}{Suppls.} \fi

\newcommand{\name}{CuMPerLay}

\renewcommand{\paragraph}[1]{{\vspace{1mm}\noindent \bf #1}.}

\unless\ifarxiv \myexternaldocument{_supplementary} \fi

\begin{document}
\title{\name: Learning Cubical Multiparameter Persistence Vectorizations}
\author{\authorBlock}
\maketitle

\begin{abstract}
We present \name, a novel differentiable vectorization layer that enables the integration of Cubical Multiparameter Persistence (CMP) into deep learning pipelines. 
While CMP presents a natural and powerful way to topologically work with images, its use is hindered by the complexity of multifiltration structures as well as the \emph{vectorization} of CMP.
In face of these challenges, we introduce a new algorithm for vectorizing MP homologies of cubical complexes. Our \name~decomposes the CMP into a combination of individual, learnable single-parameter persistence, where the bifiltration functions are jointly learned. Thanks to the differentiability, its robust topological feature vectors can be seamlessly used within state-of-the-art architectures such as Swin Transformers. We establish theoretical guarantees for the stability of our vectorization under generalized Wasserstein metrics. Our experiments on benchmark medical imaging and computer vision datasets show the benefit \name~on classification and segmentation performance, particularly in limited-data scenarios. Overall, \name~offers a promising direction for integrating global structural information into deep networks for structured image analysis. 
\end{abstract}

\section{Introduction}
\label{sec:intro}
In recent years, computer vision has witnessed remarkable advancements driven by deep learning (DL) techniques, leading to breakthroughs in classification, segmentation, and object detection. However, despite their success, DL models often struggle to effectively capture the intricate geometric and topological structures inherent in image data~\cite{cao2020comprehensive,zia2024topological,kyriakis2021learning,yang22efficient}. These limitations have spurred interest in complementary approaches that can provide a more holistic understanding of image content by incorporating global structural information beyond pixel-wise feature extraction.  

A promising direction is \textbf{topological data analysis}~\cite{carlsson2021topological} (TDA), which provides a rigorous framework for analyzing high-order structures in data. \textbf{Persistent homology} (PH), in particular, has been widely applied to extract topological descriptors that capture shape and connectivity in complex datasets~\cite{dey2022computational}. A more recent and advanced variant, \textbf{multiparameter persistence}~\cite{kaczynski2006computational} (MP), extends this approach by tracking multiple filtrations, allowing it to encode a finer-grained representation~\cite{botnan2022introduction} and enabling the extraction of richer structural features~\cite{coskunuzer2024topological}. While MP has shown considerable success in graph representation learning and point cloud analysis~\cite{demir2022todd,loiseaux2023stable}, its application to image analysis remains relatively unexplored, arguably due to the \textbf{vectorization} algorithms suited for \textbf{cubical complexes}, the native topological structure of images. 

In this work, we introduce \textbf{\name}, a novel, differentiable {vectorization} layer for {multiparameter persistence} over {cubical complexes}. While incorporating topological features from single-parameter persistence (SP) is relatively standard~\cite{carriere2020perslay,birdal2021intrinsic,surrel22ripsnet}, extending this to the cubical MP (CMP) setting is highly non-trivial. We propose to decompose the CMP into a differentiable combination of individual, learnable single filtrations. We then deploy \name~as a differentiable layer within a deep architecture, particularly Swin transformers~\cite{liu2022swin}. 
Our theoretical analysis shows that \name~shows improved stability when compared to SP counterparts.
By seamlessly integrating CMP into deep networks in an end-to-end manner, our layer helps discerning topological features essential for structured image analysis.

Our contributions can be summarized as follows:  
\begin{itemize}[noitemsep,leftmargin=*,topsep=0em]
\item We introduce a differentiable, novel vectorization for Cubical Multiparameter Persistence, combining learnable filtration functions with learnable vectorizations, and integrate into state-of-the-art image-based learning pipelines.  
\item We theoretically prove the stability of our \name\ along with a family of other filtration functions we devise.
\item We propose a hybrid DL model integrating CMP with Swin transformers, demonstrating synergistic improvements.
\item Finally, to the best of our knowledge, we develop the first (publicly available) CUDA GPU implementation of Cubical Persistence and Cubical Multiparameter Persistence.
\end{itemize}
Our experiments on benchmark medical imaging and computer vision datasets show that our topological vectors improve model performance in both standard and limited-data scenarios. 
Our sources can be found under:~\href{https://github.com/circle-group/cumperlay}{https://github.com/circle-group/cumperlay}.

\section{Related Work}
\label{sec:related}

\paragraph{TDA in image analysis} 
TDA has become a powerful tool in image processing, extracting high-level structural features invariant to transformations like rotation and noise. Early studies \cite{bendich2016persistent} applied TDA to brain MRI scans, distinguishing normal and pathological structures. More recent works \cite{clough2020topological, Xia2021,peng2024phg} integrated PH with DL, improving feature extraction and interpretability. Similarly, \cite{Zhao2022, yadav2023histopathological} showed TDA-based features enhance deep networks in histopathological image classification by capturing robust cellular representations. TDA has also improved image denoising~\cite{Sazbon2021}, demonstrating PH preserves geometric structures while filtering noise.  

Other studies used TDA as a topological loss function to enforce structural constraints during training. \cite{Hofer2019} introduced a topological signature layer for CNNs, enabling end-to-end learning of PH features, while \cite{Gabrielsson2022} proposed a differentiable topological loss to optimize structure persistence, improving medical image segmentation. The incorporation of topological features has significantly enhanced CNNs \cite{wong2021persistent, clough2020topological, Xia2021, santhirasekaram2023topology, stucki2023topologically, kwangho20pllay}. For a comprehensive review of TDA in biomedicine, see \cite{singh2023topological}. PHG-Net~\cite{peng2024phg} incorporated cubical single-persistence topological features using a PointNet~\cite{qi2017pointnet}-based encoder into recent architectures like Swin Transformers~\cite{liu2022swin} and SENet~\cite{hu2019senet}. Future research will refine topological priors, differentiable persistence computations~\cite{carriere21optimizing}, and domain-specific loss functions.

\smallskip

\noindent \textbf{Multiparameter Persistence (MP).}
MP theory enhances single persistence by capturing topological features from multidimensional filtrations, but theoretical challenges—such as the partially ordered threshold set—have hindered its adoption in machine learning (See suppl. material for detailed discussion). Most methods rely on \emph{slicing} to produce one-dimensional persistence diagrams~\cite{lesnick2015interactive, carriere2020multiparameter}, leading to information loss and dependence on slicing directions \cite{botnan2022introduction}. Alternatives extend persistence landscapes to higher dimensions, improving representation at the cost of increased complexity~\cite{vipond2020multiparameter}. Recently, \citet{loiseaux2023stable} introduced stable vectorization using signed barcodes as Radon measures, generating robust feature vectors. While effective for point clouds and graphs, these methods do not generalize to the cubical complexes considered here.  

MP has shown promise in graph representation learning \cite{loiseaux2024framework,demir2022todd} but remains underexplored in image analysis. Notable exceptions include \cite{chung2022multi}, which applied MP to mathematical morphology for topology-aware image processing and denoising, achieving results comparable to deep learning. Expanding on this, \cite{chung2024morphological} quantified mitochondrial branching morphology to study gene mutations affecting structure. While prior works focus on specific filtrations and Betti numbers, we extend MP analysis across diverse contexts and introduce flexible vectorization methods for MP modules, enabling integration into downstream DL tasks.

\section{Background}
\label{sec:bg}
We dedicate this section to exposing the less commonly known \emph{cubical complexes} and \emph{cubical multipersistence} to computer vision community with a bottom-up construction. Our exposition is inspired by~\cite{kaczynski2006computational,dey2022computational,coskunuzer2024topological}. %

\subsection{Cubical Persistence}

\begin{dfn}[Elementary interval]
    An elementary interval or a \textbf{$1$-cube}, is a closed interval $I\subset\R$: $I=[l,l+1]$ for some $l\in\Z$. We denote $[l]=[l,l]$ for a point or \textbf{0-cube}, and $I=[0,1]$ the \textbf{unit interval}. 
\end{dfn}

\begin{dfn}[Elementary cube]
    A $d$-dimensional elementary cube or \textbf{$d$-cube} $Q$ is the product of finitely many elementary intervals $Q=I_1\times I_2\times\cdots\times I_d \in \R^d$. Intuitively, the elementary cubes of a 3D grid include vertices, edges, squares ($2$-cubes) and voxels ($3$-cubes). 
\end{dfn}

\begin{dfn}[Boundary]
A boundary of an elementary interval $I$ is composed of its two endpoints: $\partial I=\partial[l,l+1]=[l+1,l+1]-[l,l]=\{l,l+1\}$. 
The \textbf{boundary of a cube} is the set of
its $(d-1)$-dimensional faces, and is a chain obtained in the following way:
\begin{equation}
    \partial Q=\bigcup_{i=1}^d \left(I_1 \times \cdots \partial I_i \times \cdots I_d\right).
\end{equation}
\end{dfn}

\begin{figure}[t] 
\centering
    \includegraphics[width=\columnwidth]{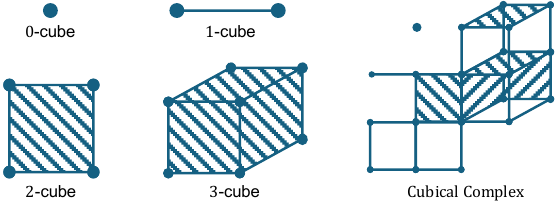} 
    \vspace{-6mm}
 {\caption{Elementary complexes of different dimensions and an exemplary cubical complex. \label{fig:CC}}}
 \vspace{-4mm}
\end{figure}

\begin{dfn}[Face]
For two elementary cubes $Q$ and $Q^\prime$, we define $Q$ to be a \textbf{face} of $Q^\prime$, denoted $Q\subseteq Q^\prime$, if $I_i\subset I_i^\prime$ for all $i=1,\dots,d$. Similarly, $P$ is said to be a \textbf{coface} of $Q$.
\end{dfn}

\begin{dfn}[Cubical complex (CC)]
An (elementary) cubical complex $\CC$ is the union of a set of $d$-cubes such that:
\begin{enumerate}[noitemsep,topsep=0em]
    \item For every $Q\in\CC$, if $P\subseteq Q$, then $P\in\CC$.
    \item If $Q\in\CC$, then all faces of $Q$ are also in $\CC$. 
\end{enumerate}
\end{dfn}
Intuitively, a CC represents a grid as a \emph{cell complex} with cells of all dimensions as illustrated in~\cref{fig:CC}.

\begin{dfn}[Cubical map]
A map $h:\CC\to\CC^\prime$ is said to be cubical if $Q\subseteq Q^\prime\in\CC \implies h(Q)\subseteq h(Q^\prime)\in\CC^\prime$.
\end{dfn}

\begin{dfn}[Chain complex of a CC]
A $n$-\textbf{chain} is a {formal finite sum}\footnote{For a set $S$ and a free Abelian group on $S$, the formal finite sum $f$ is an element of $\bigoplus_{s \in S} \mathbb{Z} = \left\{ f : S \to \mathbb{Z} \;\middle|\; \{ s \in S \mid f(s) \neq 0 \} \text{ is finite} \right\}$.} of $n$-dimensional cubes with integer coefficients. A \textbf{chain group} $C_n$ is the (\emph{free Abelian}) group generated by the set of $n$-dimensional cubes in the complex. The \textbf{cubical chain complex} $C_{\star}(\CC)$ of a CC is the sequence of chain groups and boundary operators given by:

{\small%
\setlength{\abovedisplayskip}{-1pt}%
\setlength{\belowdisplayskip}{2mm}%
\begin{align}
        \cdots \xrightarrow{\partial_{n+1}} C_n(\CC) \xrightarrow{\partial_n} C_{n-1}(\CC) \xrightarrow{\partial_{n-1}} \cdots \xrightarrow{\partial_1} C_0(\CC) \to 0.\nonumber
    \end{align}
    }
    The boundary operators satisfy $\partial_n \circ \partial_{n+1} = 0$ for all $n$.
\end{dfn}

The topological structure of a CC is characterized by identifying its holes, defined as \emph{cycles} that are not boundaries of other objects. This notion can be made rigorous as follows:

\begin{dfn}[Homology]
For a cubical chain complex, a $n$-chain $z\in C_{n}$ is called a \textbf{cycle} (closed loop) if $\partial_n(z) = 0$. Every boundary is a cycle, so  $B_n(\CC)\subseteq Z_n(\CC)$.
The set of $n$-cycles $Z_n(\CC)$ and the set of boundary elements $B_n(\CC)$ are then defined as:
\begin{align}    
    Z_n(\CC) &\vcentcolon= \ker (\partial_n) = \{ c \in C_n(\CC) \mid \partial_n(c) = 0 \}\\
    B_n(\CC) &\vcentcolon= \mathrm{im} (\partial_{n+1}) = \{ \partial_{n+1}(c) \mid c \in C_{n+1}(\CC) \}
\end{align}
Taking the group of cycles $Z_n(\CC)$ and
factoring out the cycles which act as boundaries $B_n(\CC)$ to higher dimensional cubes in the
complex, we are left with the $n^{\mathrm{th}}$ \textbf{cubical homology group}, $H_n(\CC)$, capturing all the $n$-dimensional holes:
    \begin{align}
    H_n(\CC) = {Z_n(\CC)}/{B_n(\CC)} = {\ker (\partial_n)}/{\mathrm{im} (\partial_{n+1})}.
    \end{align}
    The \textbf{homology} of $X$ is the collection of all homology groups: $H_{\star}(\CC)=\left\{H_n(\CC)\right\}_{n\in\Z}$.
\end{dfn}

To extend the notion of homology to \emph{persistent homology}, we need the concept of a filtration:
\begin{dfn}[Filtration]\label{dfn:filt}
    A filtration is a finite (or countable) sequence of nested cubical complexes (\textbf{subcomplexes}):
    \begin{align}\label{eq:filt}
        \emptyset = \CC_0\subseteq\CC_1\subseteq\cdots\subseteq\CC_{n-1}\subseteq\CC_{n}=\CC.
    \end{align}
    A \textbf{filtration function} $f_{\CC}:\CC\to\R$ describes the filtration by assigning to each $d$-cube the first index at which it appears in the complex, capturing the notion that \textbf{a cube cannot appear strictly before its faces}: $f_{\CC}(P)\leq f_{\CC}(Q),\,\forall P\subset Q$. The \textbf{sublevel set} $\CC_i:=\CC(a_i)$ corresponding to $a_i\in\R$ is defined by $f_{\CC}^{-1}(-\infty,a_i]$, a {subcomplex} of $\CC$.
    \vspace{-2mm}
\end{dfn}
\begin{figure}[t] 
\centering
\includegraphics[width=\linewidth]{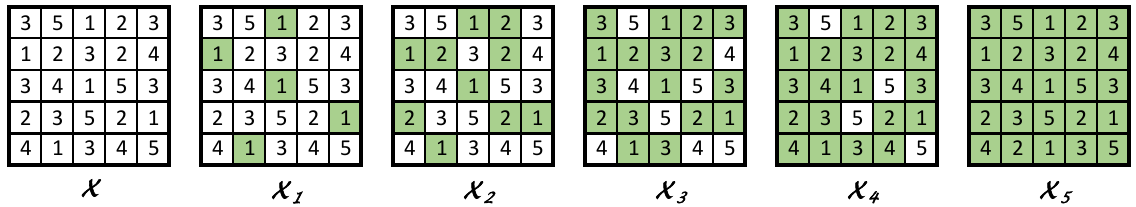}  \vspace{-5mm}
 \caption{{\bf Toy example.} For the $5\times 5$ image $\X$ with the given pixel values, \textbf{the sublevel filtration} is the sequence of binary images $\X_1\subset \X_2\subset \X_3\subset \X_4\subset \X_5$.\vspace{-4mm}}
 \label{fig:toy_filtration}
\end{figure}
\begin{example}[Threshold filtrations]Let $\X$ be a 2D grayscale image with pixels $\{\Delta_{ij} \subset \X\}$ and intensity values $I_{ij} \in [0,255]$ for an $M \times N$ grid. Given a sequence of thresholds $0 = \tau_1 < \tau_2 < \cdots < \tau_T = 255$, we obtain a nested sequence of binary images $\X_1 \subset \X_2 \subset \cdots \subset \X_T$, where $\X_T = \{\Delta_{ij} \subset \X \mid I_{ij} \leq \tau_T\}$ (see \cref{fig:toy_filtration}). This process, known as \textit{sublevel filtration}, starts with a blank image and activates pixels as their grayscale values reach the threshold $\tau_T$. Conversely, activating pixels in descending order yields a \textit{superlevel filtration}. For a color image, one can use individual color channels, e.g., $(R_{ij}, G_{ij}, B_{ij})$, for filtration. More generally, any nested sequence $\X_t \subset \X_{t+1}$ defines a filtration, independent of grayscale values or other functions.
\end{example}

In suppl. material, we present key bifiltration methods that have the potential to significantly enhance the performance of single-parameter homology in several settings. %

We are now ready to state the definitions that will lead to the construction of \emph{cubical persistent homology}. 

\begin{dfn}\label{dfn:persmod}
    Any cubical map $\CC_i\to\CC_j$ induces a linear map between the homology spaces due to \textbf{functoriality}: $\iota_{ij}:H_k(\CC_i)\to H_k(\CC_j)$ for $i\leq j$. Applying this on the filtration in~\cref{eq:filt} leads to a sequence of homology groups:
  {\small%
    \setlength{\belowdisplayskip}{2mm}%
    \begin{align}
    \label{eq:persmod}    H_k(\CC_0)\xrightarrow{\iota_{01}}H_k(\CC_1)\xrightarrow{\iota_{12}}\cdots H_k(\CC_{n-1})\xrightarrow{\iota_{n-1,n}}H_k(\CC_{n}).\nonumber
    \end{align}
    }    
\end{dfn}

\begin{dfn}[Cubical Persistent Homology (CPH)]\label{dfn:CPH}
    The sequence in \cref{dfn:persmod} is known to be the \textbf{persistence module}:
    \begin{equation}
        \mathcal{P}=\left\{H_k\left(\CC_i\right), \iota_{ij}\right\}_{0 \leq i \leq j \leq n},
    \end{equation}
    and defines the $k^\mathrm{th}$ CPH, essentially keeping track of how $k$-holes ($k$-dimensional voids) appear and disappear as we move from $\CC_0$ to $\CC_n$. CBH records these homological changes through \textbf{birth times} $b(\sigma)$ and \textbf{death times} $d(\sigma)$:
    \begin{align}
        b_{\sigma}:=b(\sigma)&:=\inf \left\{j \in I \mid \sigma \in \mathrm{Im} \iota_{ji}\right\}\\
        d_{\sigma}:=d(\sigma)&:=\inf \left\{j \in I \mid \sigma \notin \mathrm{Im} \iota_{ji}\right\}.
    \end{align}
    The \textbf{persistence} of a $k$-hole $\sigma$ refers to its \textbf{lifespan}, $d_\sigma-b_\sigma$ and the collection of half-open intervals (bars) $\{[b_\sigma,d_\sigma)\}$ is called a \textbf{persistence barcode}.
\end{dfn}

\begin{dfn}[Persistence Diagram (PD)] The $k^\mathrm{th}$ persistence diagram $\PD_k$ comprises all birth and death pairs of $k$-holes in the filtration $\CCfilt=\{\CC_i\}_{i=0}^n$:
\begin{equation}
    \PD_k(\CCfilt)=\{(b_\sigma, d_\sigma) \mid \sigma\in H_k(\CC_i) \mbox{ for } b_\sigma\leq i< d_\sigma\}. \nonumber
\end{equation}
\end{dfn}

Persistence Diagrams (PDs), consisting of collections of $2$-tuples, are not practical for utilization with ML tools. Instead, a common strategy is to transform PDs into a vector, a process referred to as \textit{vectorization}~\cite{ali2023survey}. 

\begin{dfn}[Vectorization]
    A vectorization of a PD is a map $\VectSP:\PD\to\R^M$ from the space of PDs to a fixed-dimensional Euclidean space, linearly representing the PD and allowing standard vector-based methods to be applied.
\end{dfn}

\begin{remark}[On computation]
    In practice, PDs can be efficiently computed via linear algebra as done in various efficient and robust C++ / Python libraries such as Ripser~\cite{ctralie2018ripser}, Gudhi~\cite{gudhi20}, and Giotto~\cite{tauzin2020giottotda}. Effective vectorizations can also be obtained by embeddings~\cite{adams2017persistence}, kernels~\cite{kusano2016persistence}, or neural networks as done in Perslay~\cite{carriere2020perslay}. MP for simplicial complexes can be computed through $\mathrm{multipers}$ (CPU-only). We release the first GPU implementation of MP involving CCs and share the pseudocode for our algorithm in suppl. material.
\end{remark}

We call PH applied on a single-parameter filtration, \textbf{single persistence} (SP), and next present \textbf{multipersistence} (MP).

\subsection{Cubical Multiparameter Persistence}
We now extend the definitions in the previous sections to the \emph{multiparameter} setting, constructed via a \emph{multifiltration}:
\begin{dfn}[Multifiltration]
    Given a CC $\CC$, let \(I_1, I_2, \ldots, I_m\) be totally ordered index sets, \eg $I_i = \{0,1,\ldots,N_i\}$, or intervals in $\R$. A {multifiltration} of $\CC$ is a family of subcomplexes 
    \begin{equation}\label{eq:multifilt}
     \Bigl\{ \CC_{t_1,t_2,\dots,t_m} \subseteq \CC \,\Bigr\}_{(t_1,t_2,\dots,t_m) \in I_1 \times I_2 \times \cdots \times I_m}   
    \end{equation}
    satisfying the \textbf{monotonicity condition}: if $(t_1, t_2, \dots, t_m) \le (s_1, s_2, \dots, s_m)$, \ie, $t_i \le s_i \text{ for all } i=1,\dots,m$, then $\CC_{t_1,t_2,\dots,t_m} \subseteq \CC_{s_1,s_2,\dots,s_m}$: the subcomplexes are nested simultaneously with respect to each parameter.
    If $m=2$, this is known to be a \textbf{bifiltration}. Similar to~\cref{dfn:filt}, we speak of a \textbf{bifiltration function} $F_{\CC}:\CC\to\R^2$ generating the sequence $\CC_{t_1,t_2}=F_\CC^{-1}(\Delta_{t_1,t_2})$ in~\cref{eq:multifilt} where $\Delta_{t_1,t_2}=(-\infty,t_1]\times(\infty,t_2]\subset \R^2$.
\end{dfn}

In what follows, we explicitly focus on bifiltrations of 2D images. Though, our construction extends naturally to higher dimensions, for example 3D voxel grids. For 2D images, there are numerous ways to construct bifiltrations, such as employing \emph{morphological operations}~\cite{chung2022multi} or thresholding~\cite{coskunuzer2024topological}. We give details of these bifiltrations in our suppl. material. It is important to note that determining an optimal bifiltration is a complex challenge and remains an active area of research. In this work, as we discuss in~\cref{sec:cumper}, we opt for learning bifiltrations in a data-driven manner.

We now state the biparameter persistence for CCs in a vein similar to SP:
\begin{dfn}[Cubical biparameter persistence]
Similar to~\cref{dfn:CPH}, but this time for any index pairs $(s,t),(s^\prime,t^\prime)$ with $(s,t)\leq(s^\prime,t^\prime)$ that is $s\leq s^\prime$ and $t\leq t^\prime$, the inclusion map $\CC_(s,t)\hookrightarrow\CC_(s^\prime,t^\prime)$ induces a homomorphism on homology: 
\begin{equation}
    \iota_{({s,t}),({s^\prime,t^\prime})}:H_k(\CC_{s,t})\to H_k(\CC_{s^\prime,t^\prime}),
\end{equation}
    and the collection
    \begin{equation}
        \mathcal{P}_k^\prime(\CC)=\left\{H_k\left(\CC_{s,t}\right), \iota_{(s,t),(s^\prime,t^\prime)}\right\}_{(s,t)},
    \end{equation}
    defines the cubical biparameter persistent $k$-homology. %
\end{dfn}

\begin{remark}[On multiparameter vectorizations]
    Unlike SP, there is no natural and precise extension of PD (or its vectorization) to MP outputs, primarily due to the \textbf{partial ordering problem} (See suppl. material). In a single filtration, the sequence $\{\CC_n\}$ is fully ordered and the region where a topological feature $\sigma$ persists would form a clear 1D interval $\I_\sigma$, easily expressed by $(b_\sigma, d_\sigma)$. In multifiltrations, this becomes highly complicated as the shape of the region $\I_\sigma$ might not resemble a perfect rectangle~\cite{botnan2022introduction,loiseaux2024framework}. Moreover, ordering a multifiltration $\{\CC_{s,t}\}$ for every pair of indices is far from trivial, \eg, $\CC_{2,5}$ is not comparable to $\CC_{3,4}$. Even with assumptions like barcode decomposability, expressing $\I_\sigma$ with birth and death indices becomes unfeasible unless $\I_\sigma$ is a perfect rectangle. 
\end{remark}
   On the other hand, there are readily available invariants for a multipersistence module. For example, for any pair $(s,t)$ and $(s^\prime,t^\prime)$ the rank of the map $\iota_{({s,t}),({s^\prime,t^\prime})}:H_k(\CC_{s,t})\to H_k(\CC_{s^\prime,t^\prime})$ is well-defined, and called {\em rank invariant} of the bipersistence module~\cite{botnan2022introduction}.  Similarly, $\mathcal{H}:I_1\times\dots I_2\to \mathbb{Z}$ with $\mathcal{H}(s,t)=\mbox{rank}H_k(\CC_{s,t})$ is called \textit{Hilbert function} of the bipersistence module. %
Our proposed \name, which we will introduce next, is a general MP vectorization for CCs, designed to address the aforementioned problem.

\section{CuMPer: Cubical MultiPersistence}
\begin{figure*}[t] 
 \centering
     	\includegraphics[width=\linewidth]{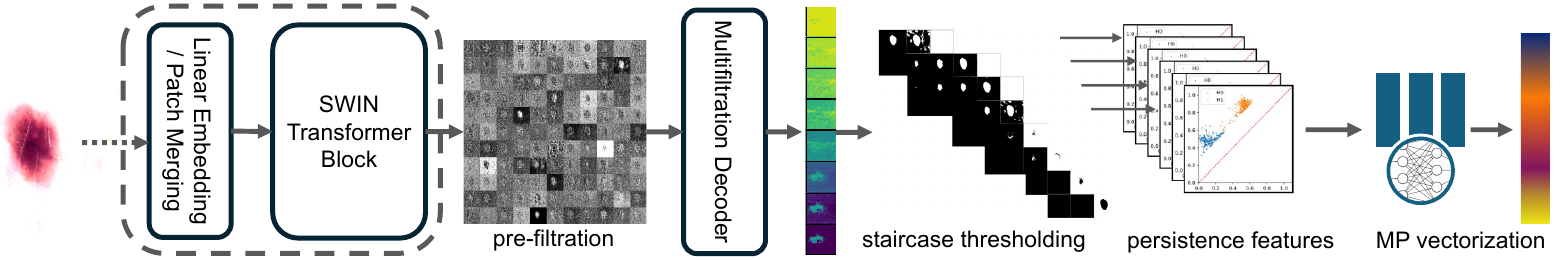}  
  \caption{{\bf \name.} Our differentiable layer can be deployed on top of existing deep networks such as Swin transformers to discern topological features of cubical complexes (\eg, images) by vectorizing multiparameter persistent homology.\vspace{-4mm}}
   \label{fig:filtration}
\end{figure*}

\subsection{Differentiable \name~Vectorization}
\label{sec:cumper}
We circumvent the aforementioned obstacles by employing a technique known as \emph{slicing}, involving the formation of horizontal slices on a MP grid of $M\times N$, where we generate a series of single-parameter filtrations denoted as $\{\CC_{s,\cdot}\}_1^M$. Next, by applying SP vectorizations to these individual filtrations, 
we derive $M$ distinct vectors. More formally:
\begin{prop}[\name] 
We introduce our cubical two-parameter vectorization $\VectMP$ to be: 
\begin{equation}\label{eq:VectMP}
    \VectMP\left(\CC\right)=\rho\left( \left[ \begin{bmatrix}
        \VectSP\left(\PD_0\left(\CC_{s,\cdot}\right)\right)\\
        \VectSP\left(\PD_1\left(\CC_{s,\cdot}\right)\right)
    \end{bmatrix}  
    \right]_t
    \right),
\end{equation}
where we choose the SP vectorization $\VectSP$ for the $s^\mathrm{th}$ row of the filtration to be a level-dependent PersLay~\cite{carriere2020perslay}:
\begin{equation}\label{eq:PersLay}
\VectSP(\PD(\CC_{\cdot,t})) :=  \mathrm{op}\left(\{w_t(p)\cdot \phi_t(p)\}_{p\in \PD(\CC_{s,\cdot})}\right).
\end{equation}
\end{prop}
While a family of choices are possible for the functions $\mathrm{op}$, $w_t$ and $\phi_t$, in this work, we make the specific choices. $\mathrm{op}$ denotes the $\mathrm{sum}$. Given weights $\mathbf{w}_t$,
$w_t(b, d\,\mid\,\mathbf{w}_t) := |d-b|^{\mathbf{w}_t}$ denotes the \emph{power weighting function} for each persistence pair. The \emph{transformation function} follows the triangle point transformation $\phi_t = \phi_{\Lambda}:\R^2\rightarrow\R^q$ sends $p\mapsto \begin{bmatrix}
	\Lambda_p(t_1), \Lambda_p(t_2), \dots, \Lambda_p(t_q)
    \end{bmatrix}^\top$, where $q\in\mathbb{N}$ and
\begin{equation}
    \Lambda_p \colon t\mapsto \max \{0, 0.5 (d-b)-|t-0.5(b+d)|\}.
\end{equation}
Finally, the aggregator $\rho:\PD^{2\times N_t}\to\R^D$ summarizes the information into a single vector.

\begin{remark}[On differentiability and learnability]
\cref{eq:VectMP} forms a \textbf{Silhouette representation}~\cite{chazal2014stochastic} that can be learned separately for each row in the MP.
In particular, the weights $\mathbf{w}_t$ as well as the transformation times $t_1,\dots,t_q\in\R$ remain learnable, as well as the parameters of the aggregation function $\rho$, which we model as a multi-layer perceptron (MLP).
We implement the Cubical PH in CUDA and propagate straight-through gradients to pixel locations that produce the corresponding persistence pairs, followed by a differentiable row-wise vectorization.
This allows us to learn the parameters of the inverse filtration functions $f_{\CC}^{-1}$, which we model as neural nets (\cref{sec:arch}) to produce $\{\CC_{s,t}\}$. 
\end{remark}

\smallskip

\noindent \textbf{Stability theorem.} \quad%
Given the technical challenges of obtaining multipersistence outputs like persistence diagrams, we circumvent this by using \textit{slicing}, a standard method in the field. By applying horizontal slices to an $M \times N$ multipersistence grid, we generate single persistence (SP) filtrations $\{\wh{\X}_i\}_1^{M}$. Using SP vectorizations, we obtain $M$ vectors $\vec{\varphi}_i=\varphi(\PD(\wh{\X}_i))$, which are then arranged as rows in a 2D tensor $\M_\varphi$, termed the induced MP vectorization of $\varphi$.

Here, we prove that if the original SP vectorization $\varphi$ is stable (see suppl. material), then so is its induced MP vectorization $\M_\varphi$ with respect to the generalized metrics. These metrics represent natural extensions of the Wasserstein metric to multipersistence modules (see suppl. material).

\begin{thm}[Stability of induced MP vectorization] \label{thm:stability} 
Let $\Vect$ be a stable SP vectorization\footnote{All SP vectorizations considered in this work are stable. We make this notion precise in our suppl. material. $\Vect$ need not necessarily correspond to the one in~\cref{eq:VectMP}.}. 
Let $\M_\Vect$ denote the \textbf{induced MP vectorization} obtained by collecting $m$ separate vectors $\Vect_t=\Vect(\PD(\CC_{s,\cdot}))$ into a $2$-tensor, $\M_\Vect$. Then, there exists a constant $\wh{C}_\varphi>0$ s.t. for any pair of CCs $\CC^1$ and $\CC^2$:
\begin{equation}\label{eq:multistable}
    \resizebox{\hsize}{!}{$\mathfrak{D}(\M_\varphi(\CC^1),\M_\varphi(\CC^2))\leq C_\varphi\cdot 
    \sum_{i=1}^m \W_{p}(\PD(\wh{\CC}_i^1),\PD(\wh{\CC}_i^-2)$},\nonumber
\end{equation}
where $\W_{p}$ denotes the $p$-\textbf{Wasserstein distance} on PDs and $\mathfrak{D}$ specifies a distance on the induced vectorizations
\begin{equation}
 \mathfrak{D}(\M_\varphi(\CC^1),\M_\varphi(\CC^2))=\sum_{i=1}^m \mathrm{d}(\varphi(\CC^1_i),\varphi(\CC^2_i)),   
\end{equation}
where $d(\cdot,\cdot)$ is a metric, usually chosen to be the $\ell_2$ distance. %
\end{thm}

As Perslay vectorizations are stable, we can use~\cref{thm:stability} to prove the stability of \name:
\begin{cor}[\name~is stable] \label{cor:cumperlay} Let $\VectMP$ represent \name~vectorization defined in~\cref{sec:cumper}. Let  $F, G:\CC\to \R^2$ be bifiltration functions on $\CC$ and $\wh{\CC}_F,\wh{\CC}_G$ be the induced bipersistence modules. Then, we have 
\begin{equation}
\mathfrak{D}(\VectMP(\wh{\CC}_F),\VectMP(\wh{\CC}_G))\leq \wh{C}_{\VectMP}\|F-G\|_\infty    
\end{equation}
\end{cor}
Our suppl. material provides the required proofs. %

\subsection{Topology Aware Learning via \name}
\label{sec:arch}
To learn multifiltration functions \emph{end-to-end}, we designed a Learnable Cubical MultiPersistence layer that can be integrated with existing deep learning networks, like common CNNs~\cite{he2016deep,szegedy2017inception,hu2019senet,liu2022convnext} and Vision Transformers~\cite{dosovitskiy2021vit,liu2022swin}. Our module consists of four main parts: (i) a multifiltration decoder, (ii) a differentiable GPU Cubical Multipersistence block, (iii) the differentiable \name~vectorization module, and (iv) a fusion block. 

\smallskip

\noindent \textbf{Multifiltration decoder.} \, The multifiltration decoder takes the input image or activations from any network layer and upsamples it using a convolutional network into the image resolution desired for persistent homology calculations. It then outputs a multifiltration in 2 dimensions using a thresholding block for enforcing the filtration structure for each row of the multifiltration and a regularization loss for enforcing the filtration structure for each column. The input to the multifiltration decoder consists of multi-channel $ \mathbf{X} \in \R^{C' \times H' \times W'} $ features, corresponding to the activations obtained from the neural network or the input image. For a MP grid of $M\times N$, the decoder outputs a $M$-channel image $ \mathbf{Z} = [\mathbf{Z}_1, \mathbf{Z}_2, \dotsc, \mathbf{Z}_M]^T \in [0,1]^{M \times H_s \times W_s} $ with $H_s, W_s$ image size desired for persistent homology process, with sigmoid activation at last layer and min-max normalization over all rows. Then, for the $s^\mathrm{th}$ row, a $N$-step \textbf{staircase function} with straight through gradients~\cite{bengio2013stgrad} is used for enforcing the multifiltration structure: 
\begin{equation}\label{eq:step-function}
    \mathbf{Z}'_s = \lfloor N \mathbf{Z}_s \rfloor \in \{0, 1, \dotsc, N\}^{H_s \times W_s}
\end{equation}
This function maps the continuous multifiltration decoder outputs to N+1 values corresponding to N values in each row in the MP grid, creating a single persistence structure per row. This process can also be viewed as a learnable alternative to sublevel filtration. We call  $ \mathbf{Z}'_s$ the compact multifiltration representation. The multifiltration $\{\CC_{s,t}\}$ can be obtained from this representation via $\{\CC_{s,t}\} = 1$ if $\mathbf{Z}'_s = t-1$, $0$ otherwise for $1 \leq t \leq N$ at each pixel location. We utilize a convolutional network with upsampling layers for the multifiltration decoder, as checkboard artifacts that can be seen in the deconvolution layers~\cite{odena2016deconvolution} can result in checkboard pattern artifacts in the multifiltration output, resulting in undesired persistence pairs. 
 
\smallskip

\noindent \textbf{Persistence and vectorization.} \, The compact multifiltration outputs are then utilized to obtain persistence pairs corresponding to each row in the multifiltration. Throughout this process, the gradients are backpropagated into pixels in the multifiltration that produce the corresponding birth and death values. Then, our differentiable \name~vectorization module produces learnable vectorizations using these persistence pairs. We use the power weighting function with learnable powers $\mathbf{w}_t$ and the triangle point transformation with learnable times $t_1,\dots,t_q\in\R$ in vectorization. We obtain the vectorizations as $H$-dimensional latent features $\mathbf{v}_s \in \R^H$ for each row $s$ in the multifiltration with the summation aggregation. These vectorizations are then concatenated into a single vector $\mathbf{v}$ corresponding to our MP vectorization in~\cref{fig:filtration}.

\begin{figure}[t] 
 \centering
   \includegraphics[width=\columnwidth]{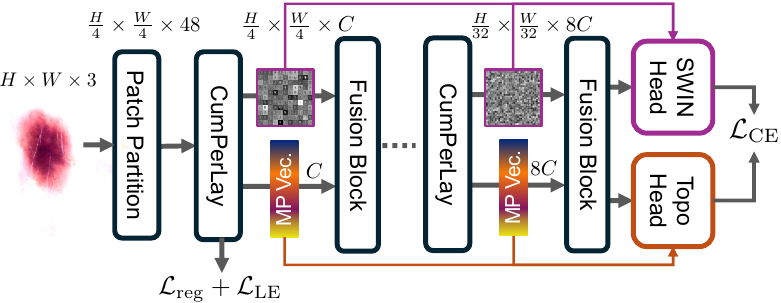}  
  \caption{{\bf TopoSwin}: Our \name~based topological Swin architecture.}
   \label{fig:toposwin}
   \vspace{-5mm}
\end{figure}

\smallskip

\noindent \textbf{Fusion block.} \, The fusion block recalibrates channel-wise activations from the original network with a residual gated linear unit~\cite{dauphin2017glu} with bias:
\vspace{-2mm}
\begin{equation}
\mathrm{FB}\left(\mathbf{X}, \mathbf{v}\right) = \mathbf{X} + \mathbf{X}\otimes \sigma\left(\mathbf{W}_1 \mathbf{v}\right) + \mathbf{W}_2 \mathbf{v},
\vspace{-2mm}
\end{equation}
where $\mathbf{X}$ is the original activations, $\mathbf{v}$ is the MP vectorization, $\mathbf{W}_i$ are the learnable weights, and $\sigma$ is the sigmoid function. We further concatenate the additional vectorization coming from a Learnable Cubical MultiPersistence layer on the input image to the vectorization $\mathbf{v}$ obtained in each layer.  

\smallskip

\noindent \textbf{Multifiltration regularization.} \quad We utilize a soft regularization loss to enforce the single filtration structure in each column. While the initial continuous output from the multifiltration decoder does not respect the bifiltration structure described before, \cref{eq:step-function} enforces this \emph{bifiltration-ness} on each row, similar to a sublevel set filtration.
The compact representations $ \mathbf{Z}'_s$ follow a structure induced by the bifiltration $\{\CC_{s,\cdot}\}$, ($ \mathbf{Z}'_s \geq  \mathbf{Z}'_{s+1}$ for $1 \leq s \leq M-1$). Our regularization loss is then formulated as:
\vspace{-2mm}
\begin{equation}
    \mathcal{L}_{\mathrm{reg}} = \sum_{s=1}^{M-1} \max(0, \mathbf{Z}'_{s+1}- \mathbf{Z}'_s)
\end{equation}
\vspace{-4mm}

This regularization guides the output of our multifiltration decoders to produce a single filtration structure on the columns. Together with the original row-wise single filtration, this ensures a multifiltration structure on our 2-parameter learned filtration. Following PHG-Net~\cite{peng2024phg}, we add a topological classification head as a two-layer MLP with a cross-entropy (CE) loss $\mathcal{L}_{\mathrm{Topo}}$ to guide the multifiltration decoders and vectorization layers, in addition to the normal cross-entropy loss $\mathcal{L}_{\mathrm{CE}}$. We use a negative local image entropy~\cite{ghosh2023localintrinsicdimensionalentropy} loss on the compact filtration outputs $\mathcal{L}_{\mathrm{LE}}$ to increase locality, promote early exploration of multifiltrations, and facilitate the training of the learnable cubical MP blocks.
The multifiltration decoder models the inverse filtration functions $f_{\CC}^{-1}$, and can be learned together with the vectorization module and the fusion block by utilizing the differentiability of the persistent homology and the vectorization steps. The layer can be repeated and used after any block in the original network (e.g. after each Swin Transformer Block in Swin V2~\cite{liu2022swin}), and gradients can be backpropagated to the network.  We give sample learned bifiltration examples in~\cref{fig:bifiltrations} and further examples in the suppl. material.

\vspace{-2mm}

\section{Experiments}
\subsection{Setup}

\noindent \textbf{Datasets.} \quad
We use three publicly available medical imaging datasets and one computer vision dataset. 
The first dataset \textbf{ISIC 2018} dataset consists of dermoscopic images of skin lesions, curated to facilitate automated melanoma classification~\cite{codella2019skin}. The dataset comprises 10,015 images spanning seven diagnostic categories. This dataset is part of the ISIC Challenge series and serves as a benchmark for skin lesion classification models. We use the official split for train, validation, and test subsets.
The second dataset \textbf{CBIS-DDSM} is a large-scale mammogram dataset for breast cancer diagnosis~\cite{lee2017curated}. It contains 3568 mammogram images, with 2111 benign and 1457 malignant cases. We utilize the official train and test split with a further 80:20 train/validation split. The third dataset \textbf{Glaucoma} is from publicly available \textit{Eye Disease Image Dataset (Eye-DID)}, a comprehensive collection of color fundus images for retinal disease classification~\cite{sharmin2024dataset,rashid2024eye}. Specifically, we utilized 1024 healthy and 1349 glaucoma fundus images for binary classification. The original 2004 × 1690 images were resized to 224×224 for computational efficiency. We used an 80:20 training/test split in this dataset, with an 80:20 split on the training portion for the validation set.
Our fourth dataset \textbf{PASCAL VOC 2012}~\cite{Everingham15} is a benchmark dataset for segmentation. We use the Semantic Boundaries~\cite{hariharan2011semantic} (SB) augmented version of the VOC dataset (excluding any VOC validation images). We use a 90:10 split on the training portion (7087 images) for the validation set. We utilize the VOC validation split as our test set (1449 images) and resize all images to 224×224.

\smallskip

\noindent \textbf{Topological baselines.} \quad To compare the standalone performance of topological models, we extract topological vectors with standard filtration functions (see suppl. material) through Multipersistence (MP) and Single Persistence (SP).

\noindent \textit{SP vectors:} We compute Betti curves for each homology group using 100 bins, resulting in 100-dimensional vectors per homology group. The concatenation of features from both homology groups is directly used for classification.

\noindent\textit{MP vectors:} For color images, we apply 10 thresholds per channel (Red, Green, and Blue), obtaining a persistence representation of shape (10 $\times$ 10 $\times$ 10) for each homology group, with two homology groups. For grayscale images, we apply 50 thresholds for grayscale and 10 erosion filtration levels [0, 1, 2, 3, 5, 7, 9, 12, 15, 20], producing an array of shape (10, 50) per homology group. These features are flattened, and the top 300 features are selected using an XGBoost~\cite{chen2016xgboost}-based feature importance ranking.

\noindent \textit{MLP Classifier:} The extracted features from both MP and SP models are used as input to a Multi-Layer Perceptron (MLP) for classification. The network consists of two hidden layers with ReLU activation and 0.3 dropout. The model is trained using the Adam optimizer with cross-entropy loss.

\begin{figure}[t] 
 \centering
     	\includegraphics[width=\linewidth]{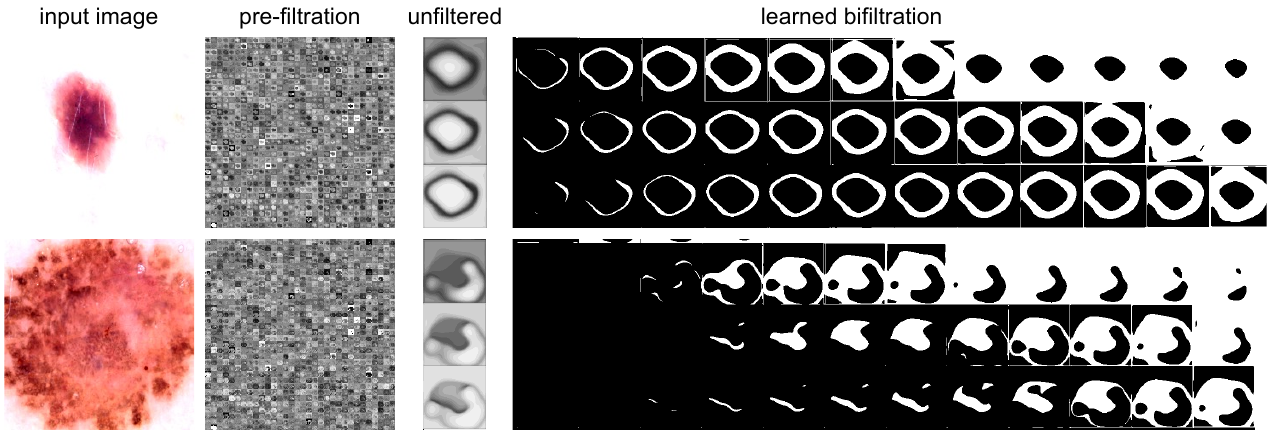}  
  \caption{\textbf{Filtration learning.} Some of the compact multifiltration representations and the corresponding bifiltrations from the third layer of our network on ISIC dataset.}
   \label{fig:bifiltrations}
   \vspace{-5mm}
\end{figure}

\paragraph{Implementation details of our TopoSwin.} \quad 
We use the SwinV2-B~\cite{liu2022swin} variant as our base model and utilize one learnable cubical multipersistence module after each Swin transformer block, with an additional module on the input images. We use 8x16 MP grids in the multifiltration decoders, with 112x112 image size for the filtration image sizes, except for the module on input images that uses 224x224.
We used 100 learnable sample locations for the triangle transformation function in our vectorization module. The base model in the Swin and TopoSwin variants is always initialized with pre-trained weights trained on the ImageNet-1K~\cite{deng2009imagenet} dataset. We do not back-propagate gradients from multifiltration decoders back into the base network when utilizing pre-trained weights.  We denote the single persistence variant of our model as TopoSwin-SP (TS-SP) and the multipersistence variant of our model as TopoSwin-MP (TS-MP).

We utilize AdamW~\cite{loshchilov2019adamw} optimizer with a learning rate of \(10^{-4}\) and weight decay of 0.01 for training our models. We do not use learning rate scheduling, use only random flip augmentations, and train for 100 epochs based on our hyperparameter tuning on the Swin baseline. We use 1.0 as $\mathcal{L}_{\text{CE}}$ weight, 0.25 as $\mathcal{L}_{\text{topo}}$ weight, and 0.01 as the $\mathcal{L}_{\text{reg}}$ weight, and \(10^{-4}\) as the $\mathcal{L}_{\text{LE}}$ weight. 

For our CUDA cubical multipersistence implementation, we implement a union-find based cubical persistence algorithm~\cite{wagner2011efficient} to efficiently find 0\textsuperscript{th} dimensional homology. We utilize the duality results between top-dimensional cell and vertex constructions of persistent homology~\cite{Bleile_2022} to calculate the 1\textsuperscript{st} dimensional homology of 2D images by utilizing the same union-find algorithm on the dual construction. Our implementation is influenced by the CPU implementation of cubical single persistence CubicalRipser~\cite{kaji2020cubicalripsersoftwarecomputing}.

\begin{table}[t]
\caption{\footnotesize \textbf{Classification results.} Performance comparison of topological and hybrid topological models on benchmark datasets. ISIC involves 7-way classification, while the others are binary classification.}
\label{tab:main-results}
\resizebox{\columnwidth}{!}{
\begin{tabular}{@{}lll|ll|ll@{}}
\toprule
\multirow{2}{*}{\textbf{Model}} & \multicolumn{2}{c}{\textbf{Glaucoma}} & \multicolumn{2}{c}{\textbf{DDSM}} & \multicolumn{2}{c}{\textbf{ISIC}} \\ 
\cmidrule(l){2-3} \cmidrule(l){4-5} \cmidrule(l){6-7} 
      & {\bf Acc}      & {\bf AUC}      & {\bf Acc}    & {\bf AUC}    & {\bf Acc}    & {\bf AUC}    \\ \midrule
{SP} & 74.32  & 81.48  &      61.65&   61.73   &          62.30       &     76.42 \\
{MP}     & 76.42 & 83.22 &     63.07 & 64.38&   63.29   &         81.16        \\
\midrule
{Swin~\cite{liu2022swin}}   & 82.73  & 92.24  & 66.76& 74.17& \underline{84.39}& 95.54\\
{PHG-Net~\cite{peng2024phg}}& 81.26  & 90.32  & 58.24& 61.07& 72.69& 89.62\\
\midrule
{TopoSwin-Betti}& \underline{83.58}  & 92.31  &   66.76   &        72.58         &  \textbf{85.01}         &  96.06    \\
{TopoSwin-SP}        & 82.95  & \textbf{92.90}   & \textbf{68.47}& \underline{75.60} & 82.34& \underline{96.12}\\
{TopoSwin-MP} & \textbf{84.63}  & \underline{92.87}  & \underline{68.32}& \textbf{76.32}& 83.47& \textbf{96.50} \\ \bottomrule
\end{tabular}
}
\vspace{-5mm}
\end{table}

\subsection{Results}
\paragraph{Medical datasets} \quad \cref{tab:main-results} presents our experimental results on the three independent medical datasets. The first two rows showcase the standalone performance of SP and MP vectors using traditional filtrations. Notably, MP consistently outperforms SP, highlighting that MP generates more informative feature vectors for image classification tasks, as expected. The third row reports the baseline deep learning (DL) model Swin Transformer V2~\cite{liu2022swin}. In the fourth row, we report PHG-Net~\cite{peng2024phg}, a hybrid topological model that integrates standard SP vectors with Swin V2. To ensure consistency, we evaluate these models using official/our data splits, with PHG-Net setup following their public code.

Finally, we introduce TopoSwin-Betti, our model incorporating MP Betti vectors with traditional filtrations. Our results indicate that the proposed hybrid models consistently enhance the performance of the Swin baseline while outperforming PHG-Net. The key distinctions between PHG-Net and our approach lie in feature extraction and learnable filtrations. While PHG-Net employs PointNet to process persistence diagrams derived from SP filtrations, our models utilize PersLay-based vectorizations with learnable filtrations.

We observe that TopoSwin-MP obtains best accuracy on Glaucoma with second-to-best accuracy on DDSM, with slightly better results from TopoSwin-SP. In addition, it obtains the best AUC on DDSM and ISIC with slightly lower AUC than TopoSwin-SP on Glaucoma. The results indicate consistent performance improvements of TopoSwin over baseline Swin model, with MP outperforming SP in many cases. In addition, we observe that TopoSwin-Betti outperforms Swin in ISIC and Glaucoma, while it underperforms in DDSM. This indicates while the standard MP features are helpful in DL scenarios, the end-to-end learnable MP features consistently provide better improvements.

\paragraph{Computer vision dataset} \quad \cref{tab:seg-results} presents our segmentation results on the VOC 2012 dataset.
\setlength{\columnsep}{6pt}
\begin{wraptable}{r}{1.8in}
\vspace{-.15in}
\caption{\footnotesize {\bf Segmentation results.} mIoU and ACC results for different baseline networks and their Topo-SP/MP variants.\vspace{-3mm}}
\label{tab:seg-results}
\setlength{\columnsep}{3pt}%
\centering
\resizebox{\linewidth}{!}{
\begin{tabular}{@{}lll|ll@{}}
\toprule
\multirow{2}{*}{\textbf{Variant}} & \multicolumn{2}{c}{\textbf{FCN}~\cite{long2015fully}} & \multicolumn{2}{c}{\textbf{U-Net}~\cite{ronneberger2015unet}} \\ 
\cmidrule(l){2-3} \cmidrule(l){4-5}
          & {\bf mIoU}    & {\bf ACC}    & {\bf mIoU}    & {\bf ACC}    \\ \midrule
{Baseline}  & 38.73 &   \underline{80.35}   &          23.96       &     78.51 \\
\midrule
{Topo-SP}  & \underline{39.25}& 80.23& \underline{26.02}& \underline{80.20}\\
{Topo-MP}& \textbf{40.36} & \textbf{80.59}& \textbf{30.42}& \textbf{81.43}\\
\bottomrule
\end{tabular}
}
\vspace{-.15in}
\end{wraptable}
We consider: (i) a Fully Convolutional Network (FCN)~\cite{long2015fully} with a ResNet-50~\cite{he2016deep} backbone pretrained on ImageNet, and (ii) a U-Net~\cite{long2015fully} trained from scratch. For Topo-SP/MP variants, we follow the same persistence setup from TopoSwin for the FCN with the ResNet-backbone, and for U-net, the same persistence modules are placed over the skip connections. In both settings, while Topo-SP outperforms the baseline, our cubical MP variant delivers consistent and larger improvements in both mIoU and Acc.

\paragraph{Performance on limited data} We hypothesize that the robustness of topological vectors can serve as effective anchors, enhancing deep learning (DL) models' ability to learn in data-scarce settings. Our experiments in limited data scenarios support this intuition, as shown in Fig. \ref{fig:limited} and suppl. material. We used the same test sets across all datasets, systematically reducing the training data to 5\%, 10\%, 20\%, 50\%, and 100\% of the original training data.
While performance variations exist across datasets, our hybrid models exhibit more stable trends as training data decreases. Overall, topological vectors do not degrade performance; instead, they contribute to more robust learning and yield performance improvements in several data-limited settings.

\begin{wraptable}{r}{1.8in}
\vspace{-.15in}
\caption{\footnotesize {\bf Ablation Study.} AUC results for different variations of our TopoSwin (TS) model.\vspace{-3mm}}
\label{tab:ablation}
\setlength{\columnsep}{3pt}%

\centering
\resizebox{\linewidth}{!}{
\begin{tabular}{@{}lcc@{}}
\toprule
{\textbf{Model}}         & \textbf{Glaucoma}& \textbf{ISIC} \\ \midrule
{Swin~\cite{liu2022swin}}& 92.24 & 95.54       \\
{TS-MP - SPL/NEnp} &     92.07 & 95.57     \\
{TS-SP - SPL} & 92.77 &	95.96\\
{TS-MP - SPL} & 92.33 &	95.60\\
{TS-Betti }       & 	92.31 & 96.06   \\
{TS-SP }    &  \textbf{92.90}  & \underline{96.12}  \\
{TS-MP}     & \underline{92.87} & \textbf{96.50}    \\
 \bottomrule
\end{tabular}}
\vspace{-.15in}
\end{wraptable}
\noindent \textbf{Ablation study.} \quad In our ablation study,  
 we assessed performance improvements by comparing TopoSwin-SP and TopoSwin-MP variants without
 \name~vectorization and image entropy loss, as well as the TopoSwin-Betti variant (\Cref{tab:ablation}). In the table, SPL denotes SP PersLay vectorization, and NEnp indicates no image entropy loss. 
  Our vectorization outperforms SP PersLay, while removing image entropy loss hinders performance,
 suggesting reduced filtration learning. These results confirm that both image entropy and our vectorization enhance Swin model performance, with learnable vectorization surpassing standard filtrations.
\begin{wraptable}{r}{1.8in}
\vspace{-.15in}
\caption{\footnotesize {\bf Comparison to persistence landscapes.} AUC results with different vectorizations. \vspace{-3mm}}
\label{tab:ablation-landscape}
\setlength{\columnsep}{3pt}%
\centering

\resizebox{\linewidth}{!}{
\begin{tabular}{@{}lcc@{}}
\toprule
\textbf{Model} & \textbf{DDSM} & \textbf{ISIC} \\ 
\midrule
{Swin~\cite{liu2022swin}}  & 74.17& 95.54\\
{TS-SP Land.}&  74.15&  95.91\\
{TS-MP Land.}&  \underline{75.60}& 95.75 \\
{TS-SP}        &   \underline{75.60} & \underline{96.12}\\
{TS-MP}   & \textbf{76.32}& \textbf{96.50} \\ 
\bottomrule
\end{tabular}
}
\vspace{-.15in}
\end{wraptable}

We also compare our PersLay-based MP vectorization with single and multi-parameter persistent landscapes~\cite{bubenik2012statistical}.
 Our PersLay-based MP vectorization subsumes many popular vectorization methods such as persistence images, silhouettes, landscapes, and kernels, and outperforms persistence landscapes (Land.) in both DDSM and ISIC in terms of AUC, tabulated on the right (\cref{tab:ablation-landscape}).
 
\begin{figure}[t]
  \centering
  \includegraphics[width=\columnwidth]{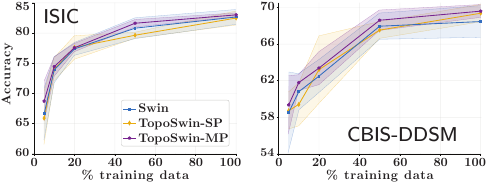}\vspace{-2mm}
  \caption{Performance comparison of Swin, TS-SP, and TS-MP models on \textbf{limited data} over 4 seeds. The $x$-axis indicates the percentage of training data used, while the test set remains unchanged.}
  \label{fig:limited}
  \vspace{-.25in}
\end{figure}

\paragraph{Runtime} \cref{tab:cpu-gpu-comparison} provides execution times (sec.) for $16\times32$ cubical MPH on $224\times 224$ images, over 300 iterations with a batch size of 8. Our GPU version (on RTX 4090) is significantly faster than Gudhi~\cite{gudhi20} implementation and its parallelization over 24 processes (on Intel i9-12900K).
\begin{table}[h!]
\vspace{-3mm}
\caption{\footnotesize {\bf Runtime comparison}. Gudhi~\cite{gudhi20} backend is used for CPU versions with \emph{CPU (Gudhi) -- 24} parallelized over 24 processes.\vspace{-3mm}}
\label{tab:cpu-gpu-comparison}
\centering
\setlength{\tabcolsep}{7pt}
\resizebox{\columnwidth}{!}{
\begin{tabular}{@{}lcc|c@{}}
\toprule
& \textbf{CPU (Gudhi~\cite{gudhi20})} & \textbf{CPU (Gudhi) -- 24} & \textbf{GPU (Ours)} \\ 
\midrule
Runtime (sec) & 41.76 $\pm$ 0.08 & 5.20 $\pm$ 1.04 & 0.57 $\pm$ 0.02 \\
\bottomrule
\end{tabular}
}
\vspace{-5mm}
\end{table}

\section{Conclusion}
\label{sec:conclusion}

We introduced \name, a novel differentiable vectorization layer for Cubical Multiparameter Persistence (CMP), effectively bridging TDA and DL, thanks to the end-to-end learnable multiparameter filtrations. By leveraging learnable MP with stable MP vectorization, we capture richer structural representations than SP and standard MP, improving feature extraction and enhancing downstream performance in SOTA DL models, particularly Swin Transformers. 

\paragraph{Limitations \& future work}
While our approach demonstrates significant promise, further research is warranted to learn optimal bifiltrations and extend the framework to higher-dimensions and more diverse datasets. We are hopeful that our work will be helpful in improving topological machine learning tool-set~\cite{hajij2022topological,papamarkou2024position,hajij2025copresheaf,hajij2023combinatorial}. 

\clearpage

\section*{Acknowledgements}
T. Birdal acknowledges support from the Engineering and Physical Sciences Research Council [grant EP/X011364/1]. T. Birdal was supported by a UKRI Future Leaders Fellowship [grant MR/Y018818/1] as well as a Royal Society Research Grant [RG/R1/241402]. This work was partially supported by National Science Foundation under grants DMS-2220613, and DMS-2229417. The authors acknowledge the \href{http://www.tacc.utexas.edu}{Texas Advanced Computing Center} (TACC) at  UT Austin for providing computational resources that have contributed to the research results reported within this paper. 
 
{\small
\bibliographystyle{ieeenat_fullname}

\begin{thebibliography}{84}
\providecommand{\natexlab}[1]{#1}
\providecommand{\url}[1]{\texttt{#1}}
\expandafter\ifx\csname urlstyle\endcsname\relax
  \providecommand{\doi}[1]{doi: #1}\else
  \providecommand{\doi}{doi: \begingroup \urlstyle{rm}\Url}\fi

\bibitem[Adams et~al.(2017)Adams, Emerson, Kirby, Neville, Peterson, Shipman, Chepushtanova, Hanson, Motta, and Ziegelmeier]{adams2017persistence}
Henry Adams, Tegan Emerson, Michael Kirby, Rachel Neville, Chris Peterson, Patrick Shipman, Sofya Chepushtanova, Eric Hanson, Francis Motta, and Lori Ziegelmeier.
\newblock Persistence images: A stable vector representation of persistent homology.
\newblock \emph{Journal of Machine Learning Research}, 18\penalty0 (8):\penalty0 1--35, 2017.

\bibitem[Ali et~al.(2023)Ali, Asaad, Jimenez, Nanda, Paluzo-Hidalgo, and Soriano-Trigueros]{ali2023survey}
Dashti Ali, Aras Asaad, Maria-Jose Jimenez, Vidit Nanda, Eduardo Paluzo-Hidalgo, and Manuel Soriano-Trigueros.
\newblock A survey of vectorization methods in topological data analysis.
\newblock \emph{IEEE Transactions on Pattern Analysis and Machine Intelligence}, 2023.

\bibitem[Atienza et~al.(2020)Atienza, Gonz{\'a}lez-D{\'\i}az, and Soriano-Trigueros]{atienza2020stability}
Nieves Atienza, Roc{\'\i}o Gonz{\'a}lez-D{\'\i}az, and Manuel Soriano-Trigueros.
\newblock On the stability of persistent entropy and new summary functions for topological data analysis.
\newblock \emph{Pattern Recognition}, 107:\penalty0 107509, 2020.

\bibitem[Bendich et~al.(2016)Bendich, Marron, Miller, Pieloch, and Skwerer]{bendich2016persistent}
Paul Bendich, James~S Marron, Ezra Miller, Alex Pieloch, and Sean Skwerer.
\newblock Persistent homology analysis of brain artery trees.
\newblock \emph{The annals of applied statistics}, 10\penalty0 (1):\penalty0 198, 2016.

\bibitem[Bengio et~al.(2013)Bengio, Léonard, and Courville]{bengio2013stgrad}
Yoshua Bengio, Nicholas Léonard, and Aaron Courville.
\newblock Estimating or propagating gradients through stochastic neurons for conditional computation, 2013.

\bibitem[Birdal et~al.(2021)Birdal, Lou, Guibas, and Simsekli]{birdal2021intrinsic}
Tolga Birdal, Aaron Lou, Leonidas~J Guibas, and Umut Simsekli.
\newblock Intrinsic dimension, persistent homology and generalization in neural networks.
\newblock \emph{Advances in neural information processing systems}, 34:\penalty0 6776--6789, 2021.

\bibitem[Bleile et~al.(2022)Bleile, Garin, Heiss, Maggs, and Robins]{Bleile_2022}
Bea Bleile, Adélie Garin, Teresa Heiss, Kelly Maggs, and Vanessa Robins.
\newblock \emph{The Persistent Homology of Dual Digital Image Constructions}, page 1–26.
\newblock Springer International Publishing, 2022.

\bibitem[Board(2020)]{gudhi20}
GUDHI~Editorial Board.
\newblock The gudhi project, 2020.
\newblock \url{https://gudhi.inria.fr/doc/3.3.0/}.

\bibitem[Botnan and Lesnick(2022)]{botnan2022introduction}
Magnus~Bakke Botnan and Michael Lesnick.
\newblock An introduction to multiparameter persistence.
\newblock \emph{arXiv preprint arXiv:2203.14289}, 2022.

\bibitem[Bubenik(2015{\natexlab{a}})]{Bubenik:2015}
P. Bubenik.
\newblock Statistical topological data analysis using persistence landscapes.
\newblock \emph{JMLR}, 16\penalty0 (1):\penalty0 77--102, 2015{\natexlab{a}}.

\bibitem[Bubenik(2015{\natexlab{b}})]{bubenik2012statistical}
Peter Bubenik.
\newblock Statistical topology using persistence landscapes.
\newblock \emph{Journal of Machine Learning Research}, 16:\penalty0 77--102, 2015{\natexlab{b}}.

\bibitem[Cao et~al.(2020)Cao, Yan, He, and He]{cao2020comprehensive}
Wenming Cao, Zhiyue Yan, Zhiquan He, and Zhihai He.
\newblock A comprehensive survey on geometric deep learning.
\newblock \emph{IEEE Access}, 8:\penalty0 35929--35949, 2020.

\bibitem[Carlsson and Vejdemo-Johansson(2021)]{carlsson2021topological}
Gunnar Carlsson and Mikael Vejdemo-Johansson.
\newblock \emph{Topological Data Analysis with Applications}.
\newblock Cambridge University Press, 2021.

\bibitem[Carri{\`e}re and Blumberg(2020)]{carriere2020multiparameter}
Mathieu Carri{\`e}re and Andrew Blumberg.
\newblock Multiparameter persistence image for topological machine learning.
\newblock \emph{{N}eur{IPS}}, 33, 2020.

\bibitem[Carri{\`e}re et~al.(2020)Carri{\`e}re, Chazal, Ike, Lacombe, Royer, and Umeda]{carriere2020perslay}
Mathieu Carri{\`e}re, Fr{\'e}d{\'e}ric Chazal, Yuichi Ike, Th{\'e}o Lacombe, Martin Royer, and Yuhei Umeda.
\newblock Perslay: A neural network layer for persistence diagrams and new graph topological signatures.
\newblock In \emph{International Conference on Artificial Intelligence and Statistics}, pages 2786--2796, 2020.

\bibitem[Carriere et~al.(2021)Carriere, Chazal, Glisse, Ike, Kannan, and Umeda]{carriere21optimizing}
Mathieu Carriere, Frederic Chazal, Marc Glisse, Yuichi Ike, Hariprasad Kannan, and Yuhei Umeda.
\newblock Optimizing persistent homology based functions.
\newblock In \emph{Proceedings of the 38th International Conference on Machine Learning}, pages 1294--1303. PMLR, 2021.

\bibitem[Chambers and Gudmundsson(2023)]{chambers2023meta}
Erin~W Chambers and Joachim Gudmundsson.
\newblock Meta-diagrams for 2-parameter persistence.
\newblock In \emph{39th International Symposium on Computational Geometry (SoCG 2023)}, page~25, 2023.

\bibitem[Chazal et~al.(2014)Chazal, Fasy, Lecci, Rinaldo, and Wasserman]{chazal2014stochastic}
Fr{\'e}d{\'e}ric Chazal, Brittany~Terese Fasy, Fabrizio Lecci, Alessandro Rinaldo, and Larry Wasserman.
\newblock Stochastic convergence of persistence landscapes and silhouettes.
\newblock In \emph{Proceedings of the thirtieth annual symposium on Computational geometry}, pages 474--483, 2014.

\bibitem[Chen and Guestrin(2016)]{chen2016xgboost}
Tianqi Chen and Carlos Guestrin.
\newblock Xgboost: A scalable tree boosting system.
\newblock In \emph{Proceedings of the 22nd acm sigkdd international conference on knowledge discovery and data mining}, pages 785--794, 2016.

\bibitem[Chung and Lawson(2019)]{chung2019persistence}
Yu-Min Chung and Austin Lawson.
\newblock Persistence curves: A canonical framework for summarizing persistence diagrams.
\newblock \emph{arXiv preprint arXiv:1904.07768}, 2019.

\bibitem[Chung et~al.(2022)Chung, Day, and Hu]{chung2022multi}
Yu-Min Chung, Sarah Day, and Chuan-Shen Hu.
\newblock A multi-parameter persistence framework for mathematical morphology.
\newblock \emph{Scientific reports}, 12\penalty0 (1):\penalty0 6427, 2022.

\bibitem[Chung et~al.(2024)Chung, Hu, Sun, and Tseng]{chung2024morphological}
Yu-Min Chung, Chuan-Shen Hu, Emily Sun, and Henry~C Tseng.
\newblock Morphological multiparameter filtration and persistent homology in mitochondrial image analysis.
\newblock \emph{Plos one}, 19\penalty0 (9):\penalty0 e0310157, 2024.

\bibitem[Clough et~al.(2020)Clough, Byrne, Oksuz, Zimmer, Schnabel, and King]{clough2020topological}
James~R Clough, Nicholas Byrne, Ilkay Oksuz, Veronika~A Zimmer, Julia~A Schnabel, and Andrew~P King.
\newblock A topological loss function for deep-learning based image segmentation using persistent homology.
\newblock \emph{IEEE transactions on pattern analysis and machine intelligence}, 44\penalty0 (12):\penalty0 8766--8778, 2020.

\bibitem[Codella et~al.(2019)Codella, Rotemberg, Tschandl, Celebi, Dusza, Gutman, Helba, Kalloo, Liopyris, Marchetti, et~al.]{codella2019skin}
Noel Codella, Veronica Rotemberg, Philipp Tschandl, M~Emre Celebi, Stephen Dusza, David Gutman, Brian Helba, Aadi Kalloo, Konstantinos Liopyris, Michael Marchetti, et~al.
\newblock Skin lesion analysis toward melanoma detection 2018: A challenge hosted by the international skin imaging collaboration (isic).
\newblock \emph{arXiv preprint arXiv:1902.03368}, 2019.

\bibitem[Coskunuzer and Ak{\c{c}}ora(2024)]{coskunuzer2024topological}
Baris Coskunuzer and C{\"u}neyt~G{\"u}rcan Ak{\c{c}}ora.
\newblock Topological methods in machine learning: A tutorial for practitioners.
\newblock \emph{arXiv preprint arXiv:2409.02901}, 2024.

\bibitem[Dauphin et~al.(2017)Dauphin, Fan, Auli, and Grangier]{dauphin2017glu}
Yann~N. Dauphin, Angela Fan, Michael Auli, and David Grangier.
\newblock Language modeling with gated convolutional networks, 2017.

\bibitem[de~Surrel et~al.(2022)de~Surrel, Hensel, Carri\`{e}re, Lacombe, Ike, Kurihara, Glisse, and Chazal]{surrel22ripsnet}
Thibault de Surrel, Felix Hensel, Mathieu Carri\`{e}re, Th\'{e}o Lacombe, Yuichi Ike, Hiroaki Kurihara, Marc Glisse, and Fr\'{e}d\'{e}ric Chazal.
\newblock Ripsnet: A general architecture for fast and robust estimation of the persistent homology of point clouds.
\newblock In \emph{Proceedings of Topological, Algebraic, and Geometric Learning Workshops 2022}, pages 96--106. PMLR, 2022.

\bibitem[Demir et~al.(2022)Demir, Coskunuzer, Gel, Segovia-Dominguez, Chen, and Kiziltan]{demir2022todd}
Andac Demir, Baris Coskunuzer, Yulia Gel, Ignacio Segovia-Dominguez, Yuzhou Chen, and Bulent Kiziltan.
\newblock Todd: Topological compound fingerprinting in computer-aided drug discovery.
\newblock \emph{Advances in Neural Information Processing Systems}, 35:\penalty0 27978--27993, 2022.

\bibitem[Deng et~al.(2009)Deng, Dong, Socher, Li, Li, and Fei-Fei]{deng2009imagenet}
Jia Deng, Wei Dong, Richard Socher, Li-Jia Li, Kai Li, and Li Fei-Fei.
\newblock Imagenet: A large-scale hierarchical image database.
\newblock In \emph{2009 IEEE Conference on Computer Vision and Pattern Recognition}, pages 248--255, 2009.

\bibitem[Dey and Wang(2022)]{dey2022computational}
Tamal~Krishna Dey and Yusu Wang.
\newblock \emph{Computational topology for data analysis}.
\newblock Cambridge University Press, 2022.

\bibitem[D{\l}otko and Gurnari(2023)]{dlotko2023euler}
Pawe{\l} D{\l}otko and Davide Gurnari.
\newblock Euler characteristic curves and profiles: a stable shape invariant for big data problems.
\newblock \emph{GigaScience}, 12:\penalty0 giad094, 2023.

\bibitem[Dosovitskiy et~al.(2021)Dosovitskiy, Beyer, Kolesnikov, Weissenborn, Zhai, Unterthiner, Dehghani, Minderer, Heigold, Gelly, Uszkoreit, and Houlsby]{dosovitskiy2021vit}
Alexey Dosovitskiy, Lucas Beyer, Alexander Kolesnikov, Dirk Weissenborn, Xiaohua Zhai, Thomas Unterthiner, Mostafa Dehghani, Matthias Minderer, Georg Heigold, Sylvain Gelly, Jakob Uszkoreit, and Neil Houlsby.
\newblock An image is worth 16x16 words: Transformers for image recognition at scale, 2021.

\bibitem[Eisenbud(2013)]{eisenbud2013commutative}
David Eisenbud.
\newblock \emph{Commutative algebra: with a view toward algebraic geometry}.
\newblock Springer Science \& Business Media, 2013.

\bibitem[Everingham et~al.(2015)Everingham, Eslami, Van~Gool, Williams, Winn, and Zisserman]{Everingham15}
M. Everingham, S.~M.~A. Eslami, L. Van~Gool, C.~K.~I. Williams, J. Winn, and A. Zisserman.
\newblock The pascal visual object classes challenge: A retrospective.
\newblock \emph{International Journal of Computer Vision}, 111\penalty0 (1):\penalty0 98--136, 2015.

\bibitem[Gabrielsson et~al.(2022)Gabrielsson, Nelson, and Carlsson]{Gabrielsson2022}
Rickard Gabrielsson, B.~J. Nelson, and Gunnar Carlsson.
\newblock A topological loss function for improved medical image segmentation.
\newblock \emph{IEEE Transactions on Medical Imaging}, 41\penalty0 (4):\penalty0 1125--1138, 2022.

\bibitem[Garin and Tauzin(2019)]{garin2019topological}
Ad{\'e}lie Garin and Guillaume Tauzin.
\newblock A topological" reading" lesson: Classification of mnist using tda.
\newblock In \emph{2019 18th IEEE International Conference On Machine Learning And Applications (ICMLA)}, pages 1551--1556. IEEE, 2019.

\bibitem[Ghosh and Motani(2023)]{ghosh2023localintrinsicdimensionalentropy}
Rohan Ghosh and Mehul Motani.
\newblock Local intrinsic dimensional entropy, 2023.

\bibitem[Hajij et~al.(2022)Hajij, Zamzmi, Papamarkou, Miolane, Guzm{\'a}n-S{\'a}enz, Ramamurthy, Birdal, Dey, Mukherjee, Samaga, et~al.]{hajij2022topological}
Mustafa Hajij, Ghada Zamzmi, Theodore Papamarkou, Nina Miolane, Aldo Guzm{\'a}n-S{\'a}enz, Karthikeyan~Natesan Ramamurthy, Tolga Birdal, Tamal~K Dey, Soham Mukherjee, Shreyas~N Samaga, et~al.
\newblock Topological deep learning: Going beyond graph data.
\newblock \emph{arXiv preprint arXiv:2206.00606}, 2022.

\bibitem[Hajij et~al.(2023)Hajij, Zamzmi, Papamarkou, Guzman-Saenz, Birdal, and Schaub]{hajij2023combinatorial}
Mustafa Hajij, Ghada Zamzmi, Theodore Papamarkou, AIdo Guzman-Saenz, ToIga Birdal, and Michael~T Schaub.
\newblock Combinatorial complexes: bridging the gap between cell complexes and hypergraphs.
\newblock In \emph{2023 57th Asilomar Conference on Signals, Systems, and Computers}, pages 799--803. IEEE, 2023.

\bibitem[Hajij et~al.(2025)Hajij, Bastian, Osentoski, Kabaria, Davenport, Dawood, Cherukuri, Kocheemoolayil, Shahmansouri, Lew, et~al.]{hajij2025copresheaf}
Mustafa Hajij, Lennart Bastian, Sarah Osentoski, Hardik Kabaria, John~L Davenport, Sheik Dawood, Balaji Cherukuri, Joseph~G Kocheemoolayil, Nastaran Shahmansouri, Adrian Lew, et~al.
\newblock Copresheaf topological neural networks: A generalized deep learning framework.
\newblock \emph{arXiv preprint arXiv:2505.21251}, 2025.

\bibitem[Hariharan et~al.(2011)Hariharan, Arbel{\'a}ez, Bourdev, Maji, and Malik]{hariharan2011semantic}
Bharath Hariharan, Pablo Arbel{\'a}ez, Lubomir Bourdev, Subhransu Maji, and Jitendra Malik.
\newblock Semantic contours from inverse detectors.
\newblock In \emph{2011 international conference on computer vision}, pages 991--998. IEEE, 2011.

\bibitem[He et~al.(2016)]{he2016deep}
Kaiming He et~al.
\newblock Deep residual learning for image recognition.
\newblock In \emph{{CVPR}}, pages 770--778, 2016.

\bibitem[Hofer et~al.(2019)Hofer, Kwitt, Niethammer, and Uhl]{Hofer2019}
Christoph Hofer, Roland Kwitt, Marc Niethammer, and Andreas Uhl.
\newblock Deep learning with topological signatures.
\newblock In \emph{Advances in Neural Information Processing Systems (NeurIPS)}, pages 16--28, 2019.

\bibitem[Hu et~al.(2019)Hu, Shen, Albanie, Sun, and Wu]{hu2019senet}
Jie Hu, Li Shen, Samuel Albanie, Gang Sun, and Enhua Wu.
\newblock Squeeze-and-excitation networks, 2019.

\bibitem[Johnson and Jung(2021)]{johnson2021instability}
Megan Johnson and Jae-Hun Jung.
\newblock Instability of the betti sequence for persistent homology and a stabilized version of the betti sequence.
\newblock \emph{Journal of the Korean Society for Industrial and Applied Mathematics}, 25\penalty0 (4):\penalty0 296--311, 2021.

\bibitem[Kaczynski et~al.(2006)Kaczynski, Mischaikow, and Mrozek]{kaczynski2006computational}
Tomasz Kaczynski, Konstantin Mischaikow, and Marian Mrozek.
\newblock \emph{Computational homology}.
\newblock Springer Science \& Business Media, 2006.

\bibitem[Kaji et~al.(2020)Kaji, Sudo, and Ahara]{kaji2020cubicalripsersoftwarecomputing}
Shizuo Kaji, Takeki Sudo, and Kazushi Ahara.
\newblock Cubical ripser: Software for computing persistent homology of image and volume data, 2020.

\bibitem[Kim et~al.(2020)Kim, Kim, Zaheer, Kim, Chazal, and Wasserman]{kwangho20pllay}
Kwangho Kim, Jisu Kim, Manzil Zaheer, Joon Kim, Frederic Chazal, and Larry Wasserman.
\newblock Pllay: Efficient topological layer based on persistent landscapes.
\newblock In \emph{Advances in Neural Information Processing Systems}, pages 15965--15977. Curran Associates, Inc., 2020.

\bibitem[Kusano et~al.(2016)Kusano, Hiraoka, and Fukumizu]{kusano2016persistence}
Genki Kusano, Yasuaki Hiraoka, and Kenji Fukumizu.
\newblock Persistence weighted gaussian kernel for topological data analysis.
\newblock In \emph{International Conference on Machine Learning}, pages 2004--2013, 2016.

\bibitem[Kyriakis et~al.(2021)Kyriakis, Fostiropoulos, and Bogdan]{kyriakis2021learning}
Panagiotis Kyriakis, Iordanis Fostiropoulos, and Paul Bogdan.
\newblock Learning hyperbolic representations of topological features.
\newblock In \emph{International Conference on Learning Representations}, 2021.

\bibitem[Lee et~al.(2017)Lee, Gimenez, Hoogi, Miyake, Gorovoy, and Rubin]{lee2017curated}
Rebecca~Sawyer Lee, Francisco Gimenez, Assaf Hoogi, Kanae~Kawai Miyake, Mia Gorovoy, and Daniel~L Rubin.
\newblock A curated mammography data set for use in computer-aided detection and diagnosis research.
\newblock \emph{Scientific data}, 4\penalty0 (1):\penalty0 1--9, 2017.

\bibitem[Lesnick(2015)]{lesnick2015theory}
Michael Lesnick.
\newblock The theory of the interleaving distance on multidimensional persistence modules.
\newblock \emph{Foundations of Computational Mathematics}, 15\penalty0 (3):\penalty0 613--650, 2015.

\bibitem[Lesnick and Wright(2015)]{lesnick2015interactive}
Michael Lesnick and Matthew Wright.
\newblock Interactive visualization of {2-D} persistence modules.
\newblock \emph{arXiv preprint arXiv:1512.00180}, 2015.

\bibitem[Liu et~al.(2022{\natexlab{a}})Liu, Hu, Lin, Yao, Xie, Wei, Ning, Cao, Zhang, Dong, et~al.]{liu2022swin}
Ze Liu, Han Hu, Yutong Lin, Zhuliang Yao, Zhenda Xie, Yixuan Wei, Jia Ning, Yue Cao, Zheng Zhang, Li Dong, et~al.
\newblock Swin transformer v2: Scaling up capacity and resolution.
\newblock In \emph{Proceedings of the IEEE/CVF conference on computer vision and pattern recognition}, pages 12009--12019, 2022{\natexlab{a}}.

\bibitem[Liu et~al.(2022{\natexlab{b}})Liu, Mao, Wu, Feichtenhofer, Darrell, and Xie]{liu2022convnext}
Zhuang Liu, Hanzi Mao, Chao-Yuan Wu, Christoph Feichtenhofer, Trevor Darrell, and Saining Xie.
\newblock A convnet for the 2020s, 2022{\natexlab{b}}.

\bibitem[Loiseaux et~al.(2023)Loiseaux, Scoccola, Carri{\`e}re, Botnan, and Oudot]{loiseaux2023stable}
David Loiseaux, Luis Scoccola, Mathieu Carri{\`e}re, Magnus~Bakke Botnan, and Steve Oudot.
\newblock Stable vectorization of multiparameter persistent homology using signed barcodes as measures.
\newblock \emph{Advances in Neural Information Processing Systems}, 36:\penalty0 68316--68342, 2023.

\bibitem[Loiseaux et~al.(2024)Loiseaux, Carri{\`e}re, and Blumberg]{loiseaux2024framework}
David Loiseaux, Mathieu Carri{\`e}re, and Andrew Blumberg.
\newblock A framework for fast and stable representations of multiparameter persistent homology decompositions.
\newblock \emph{Advances in Neural Information Processing Systems}, 36, 2024.

\bibitem[Long et~al.(2015)Long, Shelhamer, and Darrell]{long2015fully}
Jonathan Long, Evan Shelhamer, and Trevor Darrell.
\newblock Fully convolutional networks for semantic segmentation.
\newblock In \emph{Proceedings of the IEEE Conference on Computer Vision and Pattern Recognition}, pages 3431--3440, 2015.

\bibitem[Loshchilov and Hutter(2019)]{loshchilov2019adamw}
Ilya Loshchilov and Frank Hutter.
\newblock Decoupled weight decay regularization, 2019.

\bibitem[Odena et~al.(2016)Odena, Dumoulin, and Olah]{odena2016deconvolution}
Augustus Odena, Vincent Dumoulin, and Chris Olah.
\newblock Deconvolution and checkerboard artifacts.
\newblock \emph{Distill}, 2016.

\bibitem[Oudot and Scoccola(2024)]{oudot2024stability}
Steve Oudot and Luis Scoccola.
\newblock On the stability of multigraded betti numbers and hilbert functions.
\newblock \emph{SIAM Journal on Applied Algebra and Geometry}, 8\penalty0 (1):\penalty0 54--88, 2024.

\bibitem[Papamarkou et~al.(2024)Papamarkou, Birdal, Bronstein, Carlsson, Curry, Gao, Hajij, Kwitt, Li{\`o}, Di~Lorenzo, et~al.]{papamarkou2024position}
Theodore Papamarkou, Tolga Birdal, Michael Bronstein, Gunnar Carlsson, Justin Curry, Yue Gao, Mustafa Hajij, Roland Kwitt, Pietro Li{\`o}, Paolo Di~Lorenzo, et~al.
\newblock Position paper: Challenges and opportunities in topological deep learning.
\newblock \emph{arXiv preprint arXiv:2402.08871}, 2024.

\bibitem[Peng et~al.(2024)Peng, Wang, Sonka, and Chen]{peng2024phg}
Yaopeng Peng, Hongxiao Wang, Milan Sonka, and Danny~Z Chen.
\newblock Phg-net: Persistent homology guided medical image classification.
\newblock In \emph{Proceedings of the IEEE/CVF Winter Conference on Applications of Computer Vision}, pages 7583--7592, 2024.

\bibitem[Qi et~al.(2017)Qi, Su, Mo, and Guibas]{qi2017pointnet}
Charles~R. Qi, Hao Su, Kaichun Mo, and Leonidas~J. Guibas.
\newblock Pointnet: Deep learning on point sets for 3d classification and segmentation, 2017.

\bibitem[Rashid et~al.(2024)Rashid, Sharmin, Khatun, Hasan, Zahid, and Uddin]{rashid2024eye}
Riadur Rashid, Mohammad Sharmin, Shayla Khatun, Tania Hasan, Md Zahid, and Mohammad~Shorif Uddin.
\newblock Eye disease image dataset, 2024.

\bibitem[Ronneberger et~al.(2015)Ronneberger, Fischer, and Brox]{ronneberger2015unet}
Olaf Ronneberger, Philipp Fischer, and Thomas Brox.
\newblock U-net: Convolutional networks for biomedical image segmentation.
\newblock In \emph{International Conference on Medical image computing and computer-assisted intervention}, pages 234--241. Springer, 2015.

\bibitem[Santhirasekaram et~al.(2023)Santhirasekaram, Winkler, Rockall, and Glocker]{santhirasekaram2023topology}
Ainkaran Santhirasekaram, Mathias Winkler, Andrea Rockall, and Ben Glocker.
\newblock Topology preserving compositionality for robust medical image segmentation.
\newblock In \emph{{CVPR}}, pages 543--552, 2023.

\bibitem[Sazbon et~al.(2021)Sazbon, Wolf, Ben-Shachar, and Brukstein]{Sazbon2021}
D. Sazbon, Y. Wolf, O. Ben-Shachar, and L. Brukstein.
\newblock Topological data analysis for image denoising.
\newblock \emph{Pattern Recognition Letters}, 144:\penalty0 1--7, 2021.

\bibitem[Sharmin et~al.(2024)Sharmin, Rashid, Khatun, Hasan, Uddin, et~al.]{sharmin2024dataset}
Shayla Sharmin, Mohammad~Riadur Rashid, Tania Khatun, Md~Zahid Hasan, Mohammad~Shorif Uddin, et~al.
\newblock A dataset of color fundus images for the detection and classification of eye diseases.
\newblock \emph{Data in Brief}, 57:\penalty0 110979, 2024.

\bibitem[Singh et~al.(2023)]{singh2023topological}
Yashbir Singh et~al.
\newblock {TDA} in {M}edical {I}maging: Current state of the art.
\newblock \emph{Insights into Imaging}, 14\penalty0 (1):\penalty0 1--10, 2023.

\bibitem[Skraba and Turner(2020)]{skraba2020wasserstein}
Primoz Skraba and Katharine Turner.
\newblock Wasserstein stability for persistence diagrams.
\newblock \emph{arXiv preprint arXiv:2006.16824}, 2020.

\bibitem[Stucki et~al.(2023)Stucki, Paetzold, Shit, Menze, and Bauer]{stucki2023topologically}
Nico Stucki, Johannes~C Paetzold, Suprosanna Shit, Bjoern Menze, and Ulrich Bauer.
\newblock Topologically faithful image segmentation via induced matching of persistence barcodes.
\newblock In \emph{International Conference on Machine Learning}, pages 32698--32727. PMLR, 2023.

\bibitem[Szegedy et~al.(2017)Szegedy, Ioffe, Vanhoucke, and Alemi]{szegedy2017inception}
Christian Szegedy, Sergey Ioffe, Vincent Vanhoucke, and Alexander Alemi.
\newblock Inception-v4, inception-resnet and the impact of residual connections on learning.
\newblock In \emph{Proceedings of the AAAI conference on artificial intelligence}, 2017.

\bibitem[Tarjan and van Leeuwen(1984)]{tarjan84union}
Robert~E. Tarjan and Jan van Leeuwen.
\newblock Worst-case analysis of set union algorithms.
\newblock \emph{J. ACM}, 31\penalty0 (2):\penalty0 245–281, 1984.

\bibitem[Tauzin et~al.(2020)Tauzin, Lupo, Tunstall, Pérez, Caorsi, Medina-Mardones, Dassatti, and Hess]{tauzin2020giottotda}
Guillaume Tauzin, Umberto Lupo, Lewis Tunstall, Julian~Burella Pérez, Matteo Caorsi, Anibal Medina-Mardones, Alberto Dassatti, and Kathryn Hess.
\newblock giotto-tda: A topological data analysis toolkit for machine learning and data exploration, 2020.

\bibitem[Tralie et~al.(2018)Tralie, Saul, and Bar-On]{ctralie2018ripser}
Christopher Tralie, Nathaniel Saul, and Rann Bar-On.
\newblock {Ripser.py}: A lean persistent homology library for python.
\newblock \emph{The Journal of Open Source Software}, 3\penalty0 (29):\penalty0 925, 2018.

\bibitem[Vipond(2020)]{vipond2020multiparameter}
Oliver Vipond.
\newblock Multiparameter persistence landscapes.
\newblock \emph{Journal of Machine Learning Research}, 21\penalty0 (61):\penalty0 1--38, 2020.

\bibitem[Wagner et~al.(2011)Wagner, Chen, and Vu{\c{c}}ini]{wagner2011efficient}
Hubert Wagner, Chao Chen, and Erald Vu{\c{c}}ini.
\newblock Efficient computation of persistent homology for cubical data.
\newblock In \emph{Topological methods in data analysis and visualization II: theory, algorithms, and applications}, pages 91--106. Springer, 2011.

\bibitem[Wong and Vong(2021)]{wong2021persistent}
Chi-Chong Wong and Chi-Man Vong.
\newblock Persistent homology based graph convolution network for fine-grained {3D} shape segmentation.
\newblock In \emph{{ICCV}}, pages 7098--7107, 2021.

\bibitem[Xia and Wei(2021)]{Xia2021}
Kelin Xia and Guo-Wei Wei.
\newblock Persistent homology analysis of protein structure, flexibility, and folding.
\newblock \emph{International Journal for Numerical Methods in Biomedical Engineering}, 37\penalty0 (6):\penalty0 e3474, 2021.

\bibitem[Yadav et~al.(2023)Yadav, Ahmed, Daescu, Gedik, and Coskunuzer]{yadav2023histopathological}
Ankur Yadav, Faisal Ahmed, Ovidiu Daescu, Reyhan Gedik, and Baris Coskunuzer.
\newblock Histopathological cancer detection with topological signatures.
\newblock In \emph{2023 IEEE International Conference on Bioinformatics and Biomedicine (BIBM)}, pages 1610--1619. IEEE, 2023.

\bibitem[Yang et~al.(2022)Yang, Sala, and Bogdan]{yang22efficient}
Ruochen Yang, Frederic Sala, and Paul Bogdan.
\newblock Efficient representation learning for higher-order data with simplicial complexes.
\newblock In \emph{Proceedings of the First Learning on Graphs Conference}, pages 13:1--13:21. PMLR, 2022.

\bibitem[Zhao and Wang(2022)]{Zhao2022}
Xinyue Zhao and Zhiqiang Wang.
\newblock Topological data analysis for histopathological image classification using persistent homology.
\newblock \emph{Pattern Recognition Letters}, 157:\penalty0 39--46, 2022.

\bibitem[Zia et~al.(2024)Zia, Khamis, Nichols, Tayab, Hayder, Rolland, Stone, and Petersson]{zia2024topological}
Ali Zia, Abdelwahed Khamis, James Nichols, Usman~Bashir Tayab, Zeeshan Hayder, Vivien Rolland, Eric Stone, and Lars Petersson.
\newblock Topological deep learning: a review of an emerging paradigm.
\newblock \emph{Artificial Intelligence Review}, 57\penalty0 (4):\penalty0 77, 2024.

\end{thebibliography}

}

\clearpage

%\ifarxiv 
\appendix 
\section*{Appendix}

\section{Stability Results}

\subsection{Stability of Single Persistence Vectorizations} \label{sec:stability2}

In machine learning, ensuring the stability of a particular persistence diagram (PD) vectorization is very crucial, as stability examines whether small changes in the PD result in significant alterations in the vectorization. To address this question, we need clear definitions of what constitutes "small" and "large" changes, which in turn requires a way to measure distance within the space of persistence diagrams. The most widely used metric for this purpose is the Wasserstein distance, also known as the matching distance.

Consider two images, $\X^+$ and $\X^-$, with their associated persistence diagrams $\PD(\X^+)$ and $\PD(\X^-)$, respectively (omitting dimension labels). Each diagram contains points $\{q_j^+\} \cup \Delta^+$ and $\{q_l^-\} \cup \Delta^-$, where $\Delta^\pm$ represents the diagonal (which stands for trivial cycles) with infinite multiplicity. Here, $q_j^+ = (b_j^+, d_j^+)$ in $\PD(\X^+)$ denotes the birth and death times of a feature $\sigma_j$ in $\X^+$. Let $\phi: \PD(\X^+) \to \PD(\X^-)$ be a bijection, allowing mappings even when the cardinalities $|\{q_j^+\}|$ and $|\{q_l^-\}|$ differ. The $p$-th Wasserstein distance, denoted $\W_p$, is formally defined as follows:
\begin{equation}
\resizebox{\linewidth}{!}{
$\W_p(\PD(\X^+),\PD(\X^-))= \min_{\phi}\biggl(\sum_j\|q_j^+-\phi(q_j^+)\|_\infty^p\biggr)^\frac{1}{p}, \quad p\in \Z^+$.
}
\end{equation}

Next, a vectorization, denoted as $\varphi(\PD(\X))$, is called to be \textit{stable} if the distance between its outputs for two images, satisfies the inequality 

\smallskip

\noindent \resizebox{\linewidth}{!}{$\mathrm{d}(\varphi(\PD(\X^+)),\varphi(\PD(\X^-)))\leq C\cdot \W_p(\PD(\X^+),\PD(\X^-))$}

\smallskip

\noindent where $\mathrm{d}(.\ ,\ .)$ represents a suitable metric on the space of vectorizations. The constant $C>0$ is independent of the images $\X^\pm$. This stability inequality interprets that changes in vectorizations are bounded by changes in persistence diagrams. Essentially, two nearby persistence diagrams correspond to nearby vectorizations. If a particular vectorization $\varphi$ satisfy such a stability condition, we call them \textit{stable vectorization}~\cite{atienza2020stability}. Notable examples of stable vectorizations include Persistence Landscapes~\cite{Bubenik:2015}, Silhouettes~\cite{chazal2014stochastic}, Persistence Images~\cite{adams2017persistence}, Stabilized Betti Curves~\cite{johnson2021instability}, and various Persistence curves~\cite{chung2019persistence}.

\subsection{Stability of MultiPersistence Vectorizations}  \label{sec:stability3}

In this paper, because of technical problems in defining a multipersistence output like persistence diagrams (\Cref{sec:MPtheory}), we directly moved the vectorization of the multifiltrations by a well-known method called \textit{slicing}. In particular, assuming we have 2-parameter filtration $\{\X_{i,j}\}$, in the multipersistence grid $\G$ of size $m\times n$, for each fixed $i_0$, we get single filtration $\wh{\X}_{i_0}$, $\X_{i_01}\subset \X_{i_02}\subset \dots\subset \X_{i_0n}=\X_{i_0}$, which can be considered as \textit{horizontal slicing} of the multipersistence grid. Then, applying vectorizations above to the single filtrations $\{\wh{\X}_i\}_1^{m}$, we get $m$ separate vectors $\vec{\varphi}_i=\varphi(\PD(\wh{\X}_i))$. Then, collecting these vectors into a 2-tensor as rows, we obtain a 2-tensor $\M_\varphi$ which we call the induced MP vectorization of $\varphi$.
For example, when $\varphi$ is Betti vectorization, $\M_\varphi$ is simply correspond to multigraded Betti numbers~\cite{oudot2024stability}, which are simply $m\times n$ matrices for each Betti dimension. We now show that when the source single parameter (SP) vectorization $\varphi$ is stable, then so is its induced MP vectorization $\M_\varphi$.

Let $\X^+$ and $\X^-$ be two images of size $r\times s$. With the notation in \Cref{sec:stability2}, let $\varphi$ be a stable SP vectorization with the stability equation 
\begin{equation}\label{eqn1}
 \resizebox{.85\linewidth}{!}{$\mathrm{d}(\varphi(\X^+),\varphi(\X^-))\leq C_\varphi\cdot \W_{p_\varphi}(\PD(\X^+),\PD(\X^-))$}
\end{equation}
for some $1\leq p_\varphi\leq \infty$. Note that for many common stable vectorization $\varphi$,  $d(\cdot,\cdot)$ is taken as $l^{p_\varphi}$ metric~\cite{skraba2020wasserstein}. However, to keep the generality, we are not specifying it here.

Now, we consider the define these multiple single persistence diagrams $\{\PD(\wh{\X}_i^\pm)\}$ as output for the multipersistence grid, and define a natural matching distance between as the sum of the corresponding distances for each row.
\begin{equation}\label{eqn2}
 \resizebox{.85\linewidth}{!}{$\mathbf{D}_p(\{\PD(\wh{\X}_i^+)\},\{\PD(\wh{\X}_i^-)\})\\=\sum_{i=1}^m\W_p(\PD(\wh{\X}^+_i), \PD(\wh{\X}^-_i))$}.
 \end{equation} 

Now, we define \textit{the distance between induced MP vectorizations} as 
\begin{equation}\label{eqn3}
\mathfrak{D}(\M_\varphi(\X^+),\M_\varphi(\X^-))=\sum_{i=1}^m \mathrm{d}(\varphi(\X^+_i),\varphi(\X^-_i))
\end{equation}
where $p\geq 1$.
\begin{theorem} \label{thm:stability-sup}
Let $\varphi$ be a stable SP vectorization. Then, the induced MP Vectorization $\M_\varphi$ is also stable, i.e., with the notation above, there exists $\wh{C}_\varphi>0$ such that for any pair of images $\X^+$ and $\X^-$, we have the following inequality.
\begin{align*}
\resizebox{\linewidth}{!}{$\mathfrak{D}(\M_\varphi(\X^+),\M_\varphi(\X^-))\leq \wh{C}_\varphi\cdot \mathbf{D}_{p_\varphi}(\{\PD(\wh{\X}^+)\},\{\PD(\wh{\X}^-)\})$}
\end{align*}
\end{theorem}

\begin{proof} As $\varphi$ is a stable SP vectorization, for any $1\leq i\leq m$, we have $\mathrm{d}(\varphi(\X_i^+),\varphi(\X_i^-))\leq C_\varphi\cdot \W_{p_\varphi}(\PD(\X_i^+),\PD(\X_i^-))$ for some $C_\varphi>0$ by \cref{eqn1}, where 
$\W_{p_\varphi}$ is Wasserstein-$p$ distance.
Notice that the constant $C_\varphi>0$ is independent of $i$. Hence, 
\begin{align*}
\resizebox{0.4\linewidth}{!}{$\mathfrak{D}(\M_\varphi(\X^+),\M_\varphi(\X^-))$} & =  & \resizebox{0.41\linewidth}{!}{$\sum_{i=1}^m \mathrm{d}(\varphi(\X^+_i),\varphi(\X^-_i))$}\hspace{1cm}\\
\; & \leq& \resizebox{0.53\linewidth}{!}{$\sum_{i=1}^m C_\varphi\cdot \W_{p_\varphi}(\PD(\wh{\X}_i^+),\PD(\wh{\X}_i^-))$} \\
\; & =  &  \resizebox{0.53\linewidth}{!}{$C_\varphi \sum_{i=1}^m \W_{p_\varphi}(\PD(\wh{\X}_i^+),\PD(\wh{\X}_i^-))$} \\
\; &  =   & \resizebox{0.53\linewidth}{!}{$C_\varphi\cdot \mathbf{D}_{p_\varphi}(\{\PD(\wh{\X}_i^+)\},\{\PD(\wh{\X}_i^-)\})$} 
\end{align*}
where the first and last equalities are due to~\cref{eqn2} and~\cref{eqn3}, while the inequality follows from~\cref{eqn1} which is true for any $i$.   
This concludes the proof of the theorem.
\end{proof}

\begin{cor}[CumPerLay is stable] 
Let $\VectMP$ represent CumPerLay vectorization as defined in Section 4.1. Let $\CC$ be a cubical complex and let  $F, G:\CC\to \R^2$ be bifiltration functions. Let $\wh{\CC}_F,\wh{\CC}_G$ be the induced bipersistence modules. Then, we have 
$$\mathfrak{D}(\VectMP(\wh{\CC}_F),\VectMP(\wh{\CC}_G))\leq \wh{C}_{\VectMP}\|F-G\|_\infty$$
\end{cor}

\begin{proof} The proof follows from the stability of Perslay Vectorizations~\cite{carriere2020perslay}, and the previous stability theorem. In particular, while the authors proved their stability theorem for clique complexes in~\cite{carriere2020perslay}, their proofs as single persistence vectorization extends to cubical complexes as all the relevant results extends to this context~\cite{skraba2020wasserstein}. By \cite[Theorem A.2]{carriere2020multiparameter}, as single persistence module, for each row $1\leq m_0\leq M$ of the bipersistence module $\wh{\CC}$, we obtain $d_B(\PD(\wh{\CC}_F^{m_0}), \PD(\wh{\CC}_F^{m_0})\leq \|F-G\|_\infty$ where $\wh{\CC}^{m_0}$ represents $m_0^{th}$ row in bipersistence module $\wh{\CC}$. As Perslay vectorizations are stable, we obtain $$\mathrm{d}(\VectMP^{m_0}(\wh{\CC}_F),\VectMP^{m_0}(\wh{\CC}_F)\leq C_{m_0}d_B(\PD(\wh{\CC}_F^{m_0}), \PD(\wh{\CC}_F^{m_0})$$ where $\VectMP^{m_0}$ represents restriction of $\VectMP$ to $m_0^{th}$ row. After obtaining similar inequality for each row, the proof follows by following the same procedure in the proof of Thm. 1 where $\wh{C}_{\VectMP}=\sum_{m=1}^M C_m$.
\end{proof}

While we define the metrics and MP vectorizations for a 2-parameter case, it can naturally be adapted to any multiparameter case. Similarly, the proof can easily be adapted to higher dimensional images. In this paper, we utilize Silhouette and Betti curves as the vectorization method $\varphi$. while Betti curves are not stable with respect to the bottleneck ($\W_\infty$) distance~\cite{ali2023survey}, they are stable with respect to the $\W_1$-metric~\cite{dlotko2023euler}. On the other hand, Silhouette vectorizations are stable as they are derived from persistence landscapes~\cite{bubenik2012statistical,chazal2014stochastic}. Therefore, by applying Thm. 1 to these vectorizations, we have the stability for both MP Betti and MP Silhouette vectorizations, we employed in this paper. Furthermore, adapting the stability result given in \cite{oudot2024stability} to our setting, one can obtain another proof for the stability of MP Betti vectorization with respect to a signed Wasserstein-1 distance, a bottleneck-type metric they introduce.

\section{Multiparameter Persistence Theory} \label{sec:MPtheory}
Multipersistence theory has garnered significant research interest due to its potential to enhance the performance and robustness of single persistence theory. While single persistence extracts topological features from a one-parameter filtration, a multidimensional filtration with multiple parameters should, in principle, provide a richer and more informative summary for machine learning applications. However, technical challenges in the theory have hindered its full realization, leaving multipersistence largely unexplored in the ML community. Here, we summarize these key challenges. For a more detailed discussion, \cite{botnan2022introduction} provides an overview of the current state of the theory and its major obstacles.  

In single persistence, the threshold space \(\{\alpha_i\}\) is a totally ordered subset of \(\mathbb{R}\), meaning that any topological feature appearing in the filtration sequence \(\{\Delta_i\}\) has a well-defined birth and death time, with birth occurring before death. This ordering property allows for the decomposition of the persistence module \(M = \{H_k(\Delta_i)\}_{i=1}^{N}\) into barcodes, a concept formalized in the 1950s through the Krull-Schmidt-Azumaya Theorem~\cite{botnan2022introduction} (Theorem 4.2). This decomposition forms the basis of what is known as a {\em Persistence Diagram}.  

However, in higher dimensions (\(d=2\) or more), the threshold set \(\{(\alpha_i, \beta_j)\}\) is only partially ordered (a poset), meaning that while some indices have a clear ordering (e.g., \((1,2) < (4,7)\)), others do not (e.g., \((2,3)\) vs. \((1,5)\)). Consequently, in a multipersistence grid \(\{\Delta_{ij}\}\), birth and death times are no longer well-defined. Furthermore, the Krull-Schmidt-Azumaya Theorem does not extend to higher dimensions~\cite{botnan2022introduction} (Section 4.2), making barcode decomposition impossible for general multipersistence modules. This fundamental limitation prevents a straightforward generalization of single persistence to multipersistence. Even in cases where a meaningful barcode decomposition exists, the challenge remains of faithfully representing these barcodes due to the inherent partial ordering. Multipersistence modules are an active area of research in commutative algebra, with further details available in~\cite{eisenbud2013commutative}.  

Despite these challenges, several approaches have been proposed to leverage the multipersistence framework~\citep{lesnick2015theory}. One of the earliest methods, introduced by \cite{lesnick2015interactive}, involves analyzing one-dimensional slices of the multipersistence grid to extract the most dominant features. Later, \cite{carriere2020multiparameter} expanded this idea by considering multiple slicing directions (vineyards) and summarizing multiple persistence diagrams (PDs) into a vectorized representation. These slicing-based techniques extract persistence diagrams from predetermined one-dimensional slices and aggregate them into a lower-dimensional summary~\cite{botnan2022introduction}. However, this approach presents two major issues: (1) the topological summary heavily depends on the chosen slicing directions, making direction selection a nontrivial problem, and (2) compression of information from multiple persistence diagrams may lead to significant information loss.

\begin{figure}[t] 
 \centering
    \includegraphics[width=\columnwidth]{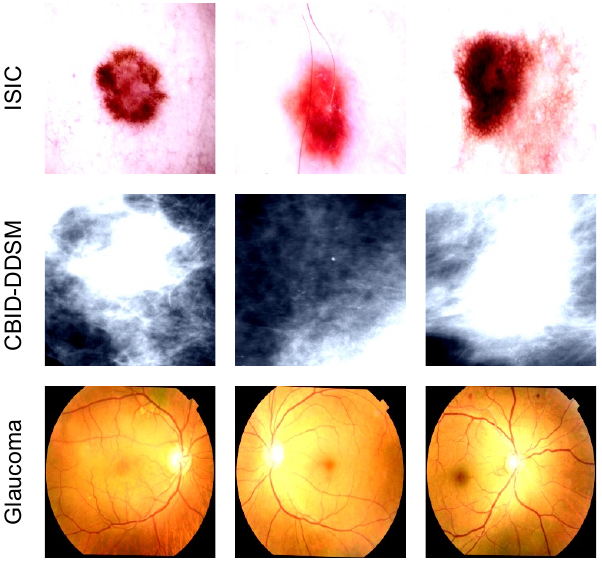}  
  \caption{\textbf{Datasets.} Example images from ISIC dataset (top row)~\cite{codella2019skin}, CBIS-DDSM dataset (middle row)~\cite{lee2017curated}, Glaucoma dataset (bottom row)~\cite{rashid2024eye,sharmin2024dataset}.}
   \label{fig:datasets}
\end{figure}

\begin{figure}[t] 
 \centering
     	\includegraphics[width=\linewidth]{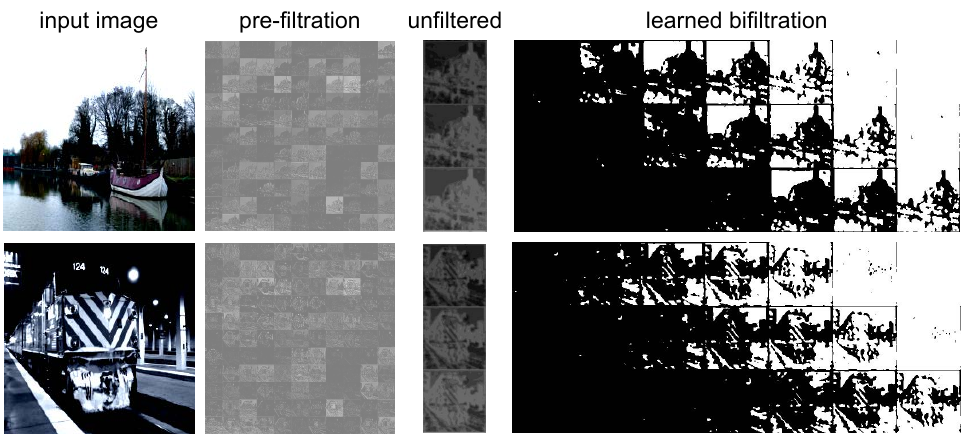}  
  \caption{\textbf{Filtration learning.} Some of the compact multifiltration representations and the corresponding bifiltrations from the first layer of our U-Net backbone Topo-MP network on PASCAL VOC 2012 dataset.}
   \label{fig:bifiltrations-voc}
   \vspace{-5mm}
\end{figure}

A different approach to vectorizing multipersistence modules was introduced by~\citet{vipond2020multiparameter}, who extended persistence landscapes into higher dimensions. Unlike slicing-based methods, this approach does not rely on a predetermined global slice direction. Instead, for each point \(\mathbf{x} \in \mathbb{R}^n\), where \(n\) is the dimension of the multipersistence module, the \(k^{\text{th}}\)-landscape explores the widest direction in which the rank invariant has a nontrivial image. From this perspective, the Multipersistence Landscape can be seen as a more faithful representation of the multipersistence module. While this method is particularly effective in settings where key insights come from a few dominant topological features (e.g., point clouds or sparse data), its computational complexity limits its practicality for large datasets with many topological features.  Recently, \cite{chambers2023meta} introduced the meta-rank and meta-diagram as novel invariants for 2-parameter persistence modules, demonstrating their equivalence to the rank invariant and signed barcode while providing computational improvements and an intuitive visualization as a persistence diagram of diagrams. Again, very recently, \cite{loiseaux2023stable} extends stable vectorization techniques from one-parameter to multiparameter persistent homology by leveraging signed barcodes as signed Radon measures, enabling the computation of informative and provably stable feature vectors that enhance performance in topology-based data analysis.

While these approaches effectively capture dominant topological patterns in sparse data, their applications are largely limited to point cloud and graph settings. In contrast, our approach offers a more computationally efficient and versatile vectorization for cubical multipersistence, where the key topological signatures mostly arise from tracking the density of small features, making it suitable for image analysis.

\begin{table*}[]
\caption{\textbf{Limited Data performances (ISIC and CBIS-DDSM).} Performance of the baseline Swin model and hybrid models incorporating Single Persistence (Swin-SP) and Multi-Persistence (Swin-MP) outputs under limited data conditions. The model was trained on X\% of the dataset (specified in the first column) while using the original test set for evaluation, with early stopping based on validation accuracy. Both datasets are evaluated and averaged over 4 different seeds, with standard deviations reported after $\pm$. Best AUC results are given in \textbf{bold}, while best accuracy results are given in \textcolor{blue}{\textbf{blue}} for each setting.}
\label{tab:limited}
\centering
\resizebox{\linewidth}{!}{
\begin{tabular}{lcc|cc|cc||cc|cc|cc}
& \multicolumn{6}{c}{{\textbf{ISIC}}} & \multicolumn{6}{c}{{\textbf{CBIS-DDSM}}} \\
\cmidrule(l){2-7} \cmidrule(l){8-13} 

& \multicolumn{2}{c}{\bf Swin~\cite{liu2022swin}}  & \multicolumn{2}{c}{\bf TopoSwin-SP*} & \multicolumn{2}{c}{\bf TopoSwin-MP*}  & \multicolumn{2}{c}{\bf Swin~\cite{liu2022swin}}  & \multicolumn{2}{c}{\bf TopoSwin-SP*} & \multicolumn{2}{c}{\bf TopoSwin-MP*} \\ 

\cmidrule(l){2-3} \cmidrule(l){4-5} \cmidrule(l){6-7}  \cmidrule(l){8-9} \cmidrule(l){10-11} \cmidrule(l){12-13}

\textbf{Data \%} & \textbf{Acc} & \textbf{AUC}& \textbf{Acc} & \textbf{AUC} & \textbf{Acc} & \textbf{AUC} & \textbf{Acc} & \textbf{AUC}& \textbf{Acc} & \textbf{AUC} & \textbf{Acc} & \textbf{AUC} \\ 
   \midrule
{0.05}&  66.62 $\pm$ 4.23  &  \textbf{87.15 $\pm$ 3.85}  &  65.92 $\pm$ 4.44  &  84.20 $\pm$ 6.75  &  \textcolor{blue}{\textbf{68.74 $\pm$ 3.32}}  &  85.94 $\pm$ 3.17  &   58.52 $\pm$ 4.44  &  \textbf{63.34 $\pm$ 1.68}  &  58.72 $\pm$ 2.03  &  62.81 $\pm$ 3.27  &  \textcolor{blue}{\textbf{59.35 $\pm$ 3.26}}  &  61.10 $\pm$ 1.62  \\
{0.10} &  73.90 $\pm$ 2.15  &  91.48 $\pm$ 0.58  &  74.01 $\pm$ 1.88  &  \textbf{91.64 $\pm$ 1.01}  &  \textcolor{blue}{\textbf{74.45 $\pm$ 1.59}}  &  90.71 $\pm$ 1.59  &  60.80 $\pm$ 2.00  &  60.60 $\pm$ 3.15  &  59.40 $\pm$ 2.37  &  61.70 $\pm$ 2.47  &  \textcolor{blue}{\textbf{61.79 $\pm$ 0.88}}  &  \textbf{64.31 $\pm$ 3.13}   \\
{0.20}&  77.38 $\pm$ 1.08  &  93.65 $\pm$ 0.38  &  \textcolor{blue}{\textbf{77.62 $\pm$ 1.94}}  &  \textbf{93.80 $\pm$ 1.17}  &  77.58 $\pm$ 0.68  &  93.07 $\pm$ 0.28  &  62.50 $\pm$ 1.82  &  69.04 $\pm$ 1.16  &  63.21 $\pm$ 3.69  &  \textbf{69.08 $\pm$ 0.96}  & \textcolor{blue}{\textbf{63.38 $\pm$ 1.81}}  &  67.25 $\pm$ 0.83    \\
{0.50}  &  80.82 $\pm$ 1.73  &  \textbf{95.46 $\pm$ 0.83}  &  79.65 $\pm$ 0.50  &  94.80 $\pm$ 0.65  &  \textcolor{blue}{\textbf{81.64 $\pm$ 0.62}}  &  95.33 $\pm$ 0.89  &  67.95 $\pm$ 1.46  &  \textbf{74.92 $\pm$ 1.24}  &  67.53 $\pm$ 0.94  &  73.22 $\pm$ 1.85  &  \textcolor{blue}{\textbf{68.61 $\pm$ 1.13}}  &  74.46 $\pm$ 1.36  \\
{1.00}   &  82.67 $\pm$ 1.26  &  95.62 $\pm$ 0.15  &  82.50 $\pm$ 1.17  &  95.63 $\pm$ 0.71  &  \textcolor{blue}{\textbf{83.02 $\pm$ 0.35}}  &  \textbf{95.93 $\pm$ 0.55}  &  68.47 $\pm$ 1.76  &  75.53 $\pm$ 0.77  &  69.38 $\pm$ 0.97  &  75.89 $\pm$ 0.97  &  \textcolor{blue}{\textbf{69.60 $\pm$ 0.72}}  &  \textbf{76.84 $\pm$ 0.94}   \\
\bottomrule
\end{tabular}}
\end{table*}

\begin{table}[]
\caption{\textbf{Limited Data performances (Glaucoma).} Performance of the baseline Swin model and hybrid models incorporating Single Persistence (Swin-SP) and Multi-Persistence (Swin-MP) outputs under limited data conditions. The model was trained on X\% of the dataset (specified in the first column) while using the original test set for evaluation. Best AUC results are given in \textbf{bold}, while best accuracy results are given in \textcolor{blue}{\textbf{blue}} for each setting.}
\label{tab:limited-glaucoma}
\centering
\resizebox{\linewidth}{!}{
\begin{tabular}{lcc|cc|cc}
& \multicolumn{6}{c}{{\textbf{Glaucoma}}} \\
\cmidrule(l){2-7}

& \multicolumn{2}{c}{\bf Swin~\cite{liu2022swin}}  & \multicolumn{2}{c}{\bf TopoSwin-SP*} & \multicolumn{2}{c}{\bf TopoSwin-MP*} \\ 

\cmidrule(l){2-3} \cmidrule(l){4-5} \cmidrule(l){6-7} 

\textbf{Data \%} & \textbf{Acc} & \textbf{AUC}& \textbf{Acc} & \textbf{AUC} & \textbf{Acc} & \textbf{AUC}\\ 
   \midrule
{0.05} & \textcolor{blue}{\textbf{79.47}}  & 85.30  & 67.57  & 72.76  & 78.10  & \textbf{88.04}  \\
{0.10} & 78.52  & \textbf{89.42}    & 73.26  & 89.33  & \textcolor{blue}{\textbf{78.74}}  & 89.05  \\
{0.20} & \textcolor{blue}{\textbf{82.95}}  & \textbf{92.00}   & 82.31  & 91.34  & 82.52  & 90.63  \\
{0.50}  & 81.68  & 90.94   & \textcolor{blue}{\textbf{83.57}}  & 91.63  & 81.89  & \textbf{92.64}  \\
{1.00}    & 82.73  & 92.24    & 82.95  & \textbf{92.90}   & \textcolor{blue}{\textbf{84.63}}  & 92.87  \\
\bottomrule
\end{tabular}}
\end{table}

\begin{figure}[t]
    \centering
   \includegraphics[width=\linewidth, keepaspectratio]{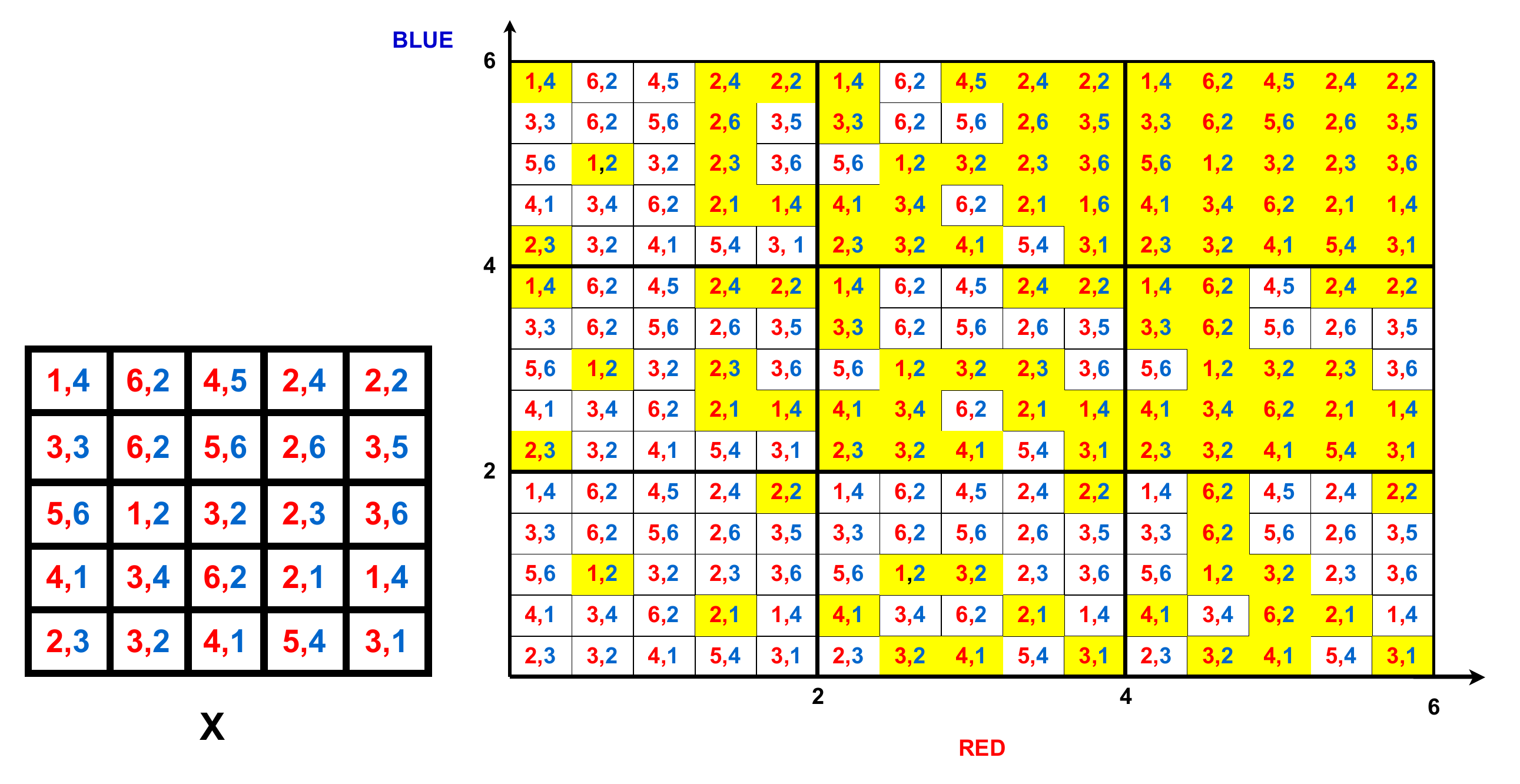}\vspace{-3mm}
    \caption{\footnotesize {\bf Toy illustration of bifiltration.} For a given image $X$ with two color channels, our color multifiltration produces 9 binary images in $3\times 3$ grid. In horizontal directions, we activate (coloring yellow) the pixels whose red color value is below the given threshold. In vertical direction, we only consider blue color values to activate the pixels. 
    }
    \label{fig:bifiltration-toy}
    \vspace{-4mm}
\end{figure}

\section{Common Multifiltrations for Images} \label{sec:bifiltrations}

There are several natural multifiltration methods for cubical persistence the images offer. Here, we list some of them.

\paragraph{Bifiltrations on binary images}
    Let $\X$ be an image of size $r\times s$. A bifiltration induced from $\X$ is a collection of binary images $\{\X_{m,n}\}$ of the same size where $\X_{m,n}\subset \X_{m+1,n}$ and $\X_{m,n}\subset \X_{m,n+1}$ for any $1\leq m\leq M$ and $1\leq n\leq N$. In other words, a bifiltration defines a grid $\G$ of size $M\times N$ where each individual row and column represents a regular filtration as shown in~\cref{fig:bifiltration-toy}, \ie, for any fixed $1\leq m_0 \leq M$ (row), the sequence $\{\X_{m_0,n}\}_1^N$ is a regular filtration of length $N$ (column), and for any fixed $1\leq n_0 \leq N$, the sequence $\{\X_{m,n_0}\}_1^M$ is a regular filtration of length $M$.

\paragraph{Color multifiltrations.} The primary objective of multiparameter persistence lies in effectively leveraging multiple parameters, especially when data offers more than one function to utilize. In color images, each image inherently comprises three natural functions, denoted as $R$, $G$, and $B$. Thus, for every pixel $\Delta_{ij}$, there exist corresponding color values: $R_{ij},G_{ij},B_{ij}\in [0,255]$. To proceed, we establish a three-parameter multifiltration with parameters $\{s_m\}_1^{N_1}$, $\{t_n\}_1^{N_2}$, $\{v_r\}_1^{N_3}$, where $s_m,t_n,v_r\in [0,255]$ are threshold values for color channels $R,G$, and $B$ respectively. By simply defining binary images $\X_{m,n,r}=\{\Delta_{ij}\subset \X \mid  R_{ij}\leq s_m, G_{ij}\leq t_n, B_{ij}\leq v_r\}$, we induce a three-parameter multipersistence module $\{H_k(\X_{m,n,r})\}$, yielding Betti tensors $\B_k(\X)=[\beta^k_{m,n,r}]$. These tensors constitute $3D$ arrays with dimensions $N_1 \times N_2 \times N_2$ (See Fig. 3).

\begin{figure}[b]
\centering
\includegraphics[width=\linewidth]{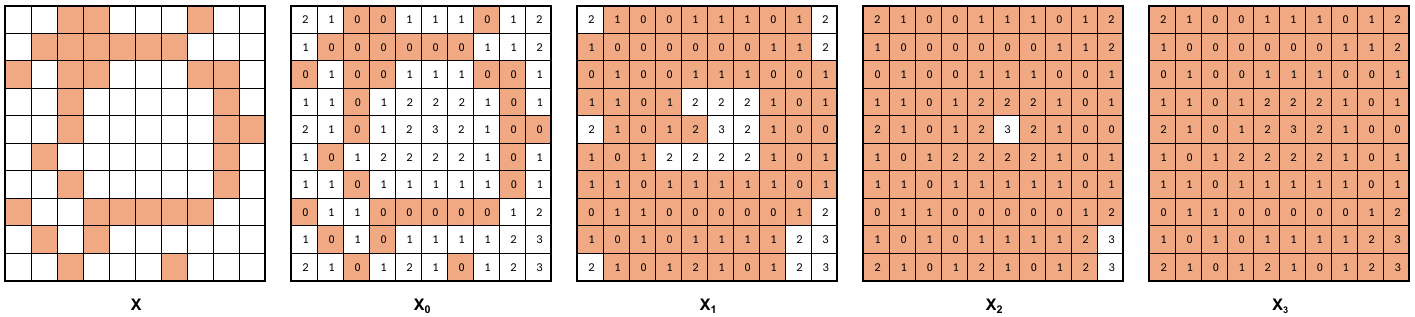}
\caption{\footnotesize {\bf Erosion Filtration.} For a given binary image $\X$, we first define the erosion function (given in $\X_0$). Then, we obtain the filtration of binary images $\X_0\subset\X_1\subset \dots \X_3$, by activating the pixels reaching the threshold value. \label{fig:erosion}}
\end{figure}

\paragraph{Erosion bifiltrations.} One significant limitation of sublevel filtration is its inability to provide information about the sizes of topological features. Instead, it only focuses on the difference in function values between when a topological feature is born and when it dies. To illustrate, consider a grayscale image where all pixels have a grayscale value of $0$ except for one pixel in the center with a value of $255$. The resulting persistence diagram would feature a single large bar $[0,255)$, indicating a very small topological feature—a hole with a diameter of $1$. Conversely, a binary image in $\X_{100}$ might contain a large hole with a diameter of 20, where the pixels of the hole have color values between 101 and 105, then it will be completely filled in by $\X_{105}$. Despite the significant change in the hole's size, the grayscale sublevel filtration would only yield a small bar $(100,105)$ reflecting the difference in grayscale values (color contrast of the hole), without conveying any information about the hole's size.

While persistence homology aims to identify the topological features present in a filtration, sublevel filtrations cannot inherently capture size information for these features. However, alternative filtration methods such as erosion, dilation, and signed distance filtrations have been proposed to address this issue~\cite{garin2019topological}. These methods offer avenues to incorporate size information into the analysis of topological features, complementing the capabilities of sublevel filtration. %

To capture topological features created by the images as well as their sizes, a natural approach is to combine grayscale sublevel filtration with erosion filtration. Erosion filtration is defined for binary images, and basically, it is a sublevel filtration for a special function, called {\em erosion}. Let $\X$ be a grayscale image of size $r\times s$, and let $\{\X_m\}_1^M$ be the sublevel filtration induced by grayscale values described in Section 3. Each $\X_m$ is a binary image with some topological features. Let $\Omega_m$ represent all black pixels in $\X_m$ and $\G$ is $r\times s$ grid. Then, for each $\X_m$, we define an erosion function $\xi_m:\G\to \N$ such that $\xi_m(i,j)=\D(\Omega_m, \Delta_{ij})$ where $\D$ is the Manhattan ($L_1$) metric between the pixel $\Delta_{ij}$ and the black region $\Omega_m$. In other words, $\xi_m(i_0,j_0)=\inf\{|i_0-i'|+|j_0-j'|\mid \Delta_{i'j'}\subset \Omega_m\}$. In particular, for each $\Delta_{ij}\in \Omega_m$, $\xi_m(i,j)=0$. Intuitively, erosion function $\xi_m$ gives $0$ values to all black pixels in the binary image $\X_m$, and measures the $L_1$ distance of each white pixel to black region $\Omega_m$ in $\X_m$ (See~\cref{fig:erosion}). Then, for each $m$, the bifiltration $\{\X_{m,n}\}$ is defined as the collection of binary images $\X_{m,n}=\{\Delta_{ij}\in \X \mid \xi_m(i,j)< n\}$. Hence, if there is a white hole in $\X_m$, the value $d$ of the farthest pixel in the hole to $\Omega_m$ measures the radius of the whole, and produces a bar of size $[0,d)$ in the persistence diagram. In a way, erosion filtration is enabling to measure the size of the wholes in binary images. Recall that the persistence bars for sublevel filtration for color values only give the color difference (contrast) of the hole independent of the size of the hole. Therefore, color filtration and erosion filtration are completely complementary to each other, and combining them produces a very powerful filtration model for the image. Signed distance and dilation are other types of filtrations to capture the size of the topological features in binary images, similar to erosion filtration~\cite{garin2019topological}.

\section{Implementation Details}
In \cref{alg:cubical_persistence}, we present the main algorithm for our CUDA GPU implementation of Cubical Persistence and Cubical Multiparameter Persistence. 
The two main C++/CUDA functions are further explained in \cref{alg:enum_edges} and \cref{alg:joint_pairs}. 
Our implementation is influenced by the CPU-based CubicalRipser software~\cite{kaji2020cubicalripsersoftwarecomputing} for SP. 
We utilize the duality between top-dimensional cell and vertex constructions in persistent homology~\cite{Bleile_2022} to efficiently compute 0\textsuperscript{th} and 1\textsuperscript{st} dimensional homology of 2D images with only minor modifications to the algorithm. 
The compact multifiltration inputs (corresponding to $N$ MP filtration inputs with $C$ row-wise slices) are used to construct an input grid whose structure depends on the target dimension.
For the 0\textsuperscript{th} dimension, the compact filtration input is padded with a threshold value (a constant input representing infinity), corresponding to a primal grid. For the 1\textsuperscript{st} dimension, the input values are padded with threshold, inverted, and then padded again, corresponding to an Alexander Duality-based dual grid.
Following this, \cref{alg:enum_edges} iterates over all edges (horizontal and vertical for both dimensions, plus diagonal edges for the 1\textsuperscript{st} dimension), computes filtration values based on the compact filtration inputs in parallel over the batch, width, and height dimensions, and finally sorts the values in a descending order.

In \cref{alg:joint_pairs}, we use a union-find-based algorithm with a CUDA-adapted implementation of Union-Find~\cite{tarjan84union} with path compression and union by rank. The initialization of the union-find data structure is parallelized over each batch element and each pixel.
The main loop, which calculates persistence pairs by iterating over edges sorted by decreasing death times, is parallelized only over the batch dimension.
Since the number of persistence pairs can vary per input and is not known beforehand, we use a chunking mechanism. Based on a provided chunk size hyperparameter, the loop is executed until a maximum of chunk size pairs are generated for each batch element, while the number of persistence pairs is tracked separately. The kernel for this loop is executed iteratively until the computation for each batch element is complete. The resulting chunks are then concatenated and processed in parallel. This approach allows variable-length persistence pair outputs and can be further optimized by fine-tuning the chunk size hyperparameter for a given task.

Finally, we extract the final persistence pairs using the coordinate locations that correspond to the birth and death values in the original MP filtration.
The algorithm's output for variable-length persistence pairs consists of two tensors: a zero-padded tensor containing the persistence pairs for each batch element and row-wise slice, and a separate tensor storing the lengths for each batch element and row-wise slice. 
This output is then processed by our \name~vectorization layer, which takes masked persistence pair tensors as inputs.
While we typically use a chunk size of 1024, this hyperparameter can be dynamically adjusted during training based on the maximum number of persistence pairs observed in recent batches. 

 \begin{algorithm}[t]
 \SetNoFillComment
    \caption{Cubical Persistence for 2D Images}
    \label{alg:cubical_persistence}
    \SetKwInOut{KwIn}{Input}
    \SetKwInOut{KwOut}{Output}
    \KwIn{Image filtration batches $I_{batch} \in \mathbb{R}^{N \times C \times H \times W}$ ($C$ rows), \\threshold $T$}
    \KwOut{Persistence pairs $P$, Lengths $L$}

    $I \leftarrow \text{Reshape}(I_{batch}, (N \cdot C, H, W))$\;
    $P_{list}, L_{list} \leftarrow [], []$\;

    \For{$d \in \{0, 1\}$}{
        \tcp{Prepare grid for dimension $d$}
        $G \leftarrow \text{PadGrid}(I, \text{value}=T, d)$
        
        \tcp{Compute persistence}
        $V, Idx \leftarrow \text{EnumerateAndSortEdges}(G, d)$\;
        $C, L_d \leftarrow \text{JointPairs}(G, V, Idx, d)$ %

        $P_d \leftarrow \text{ExtractPersistenceValues}(C, I, d)$\;

        $P_{list}.\text{append}(P_d)$\;
        $L_{list}.\text{append}(L_d)$\;
    }

    \tcp{Combine and reshape results}
    $P \leftarrow \text{PadAndStack}(P_{list})$\;
    $L \leftarrow \text{Stack}(L_{list})$\;
    $P, L \leftarrow \text{ReshapeToOriginal}(P, L, (N, C, \dots))$\;
    
    \KwRet{$P, L$}\;
\end{algorithm}
\begin{algorithm}[t]
\SetNoFillComment
    \caption{EnumerateAndSortEdges}
    \label{alg:enum_edges}
    \SetKwInOut{KwIn}{Input}
    \SetKwInOut{KwOut}{Output}
    \KwIn{Grid $G \in \mathbb{R}^{B \times H \times W}$, dimension $d \in \{0, 1\}$}
    \KwOut{Sorted filtration values $V$,\\ Sorted indices $Idx$}
    
    \tcp{Define edge connectivity}
    \uIf{$d = 0$}{
        \tcp{Horizontal and vertical edges}
        $\text{Offsets} \leftarrow \{(1, 0), (0, 1)\}$ 
    }
    \Else{
        \tcp{Include diagonal edges}
        $\text{Offsets} \leftarrow \{(1, 0), (0, 1), (1, 1), (1, -1)\}$ 
    }

    $E_{list} \leftarrow []$ %

    \tcp{Compute filtration values (CUDA)}
    \For{each image in batch $G$ \textbf{in parallel}}{
        \For{each pixel $v=(x,y)$ in image \textbf{in parallel}}{
            \For{each offset $o \in \text{Offsets}$}{
                $u \leftarrow v + o$\;
                $\text{value} \leftarrow \max(G(v), G(u))$\;
                $E_{list}.\text{append}(\text{value})$\;
            }
        }
    }

    \tcp{Sort edges by filtration values}
    $V, Idx \leftarrow \text{SortAndGetIndices}(E_{list}, \text{order}=\text{descending})$\;
    
    \KwRet{$V, Idx$}\;
\end{algorithm}

\begin{algorithm}[h]
\SetNoFillComment

    \caption{JointPairs}
    \label{alg:joint_pairs}
    \SetKwInOut{KwIn}{Input}
    \SetKwInOut{KwOut}{Output}
    \KwIn{Grid $G \in \mathbb{R}^{B \times H \times W}$, Sorted filtr. values $V$, \\Sorted indices $Idx$, dimension $d \in \{0, 1\}$}
    \KwOut{Paired cell coordinates $C$, Lengths $L$}
    
    \tcp{Initialize components for pairing}
    $UF \leftarrow \text{InitializeUnionFind()}$\;
    \For{each vertex $v \in G$ \textbf{in parallel}}{
        $UF.\text{Add}(v, \text{birth}=\text{GetVertexBirth}(G, v))$\;
    }
    
    $C_{list} \leftarrow []$\;
    
    \tcp{Parallel over batch $B$ (CUDA)}
    \For{each edge $e$ in the sorted filtration $(V, Idx)$}{
        $u, v \leftarrow \text{GetVerticesOfEdge}(e)$\;
        $root_u, root_v \leftarrow UF.\text{Find}(u), UF.\text{Find}(v)$\;
        
        \If{$root_u \neq root_v$}{
            $\text{birth}_u, \text{birth}_v \leftarrow UF.\text{GetBirth}(root_u), UF.\text{GetBirth}(root_v)$\;
            
            \uIf{$\text{birth}_u \geq \text{birth}_v$}{
                $C_{list}.\text{append}((\text{coord}(root_u), \text{coord}(e)))$\;
                $UF.\text{Union}(root_u, root_v, \text{birth}=\text{birth}_v)$\;
            }
            \Else{
                $C_{list}.\text{append}((\text{coord}(root_v), \text{coord}(e)))$\;
                $UF.\text{Union}(root_u, root_v, \text{birth}=\text{birth}_u)$\;
            }
        }
    }
    
    \If{$d=0$}{
        $root_{essential} \leftarrow \text{GetFinalComponent}(UF)$\;
        $C_{list}.\text{append}((\text{coord}(root_{essential}), \text{coord}_{inf}))$\;
    }
    $C, L \leftarrow \text{FormatAndPack}(C_{list})$\;
    
    \KwRet{$C, L$}\;
\end{algorithm}
\clearpage

\begin{figure*}[h]
\centering
\hspace{0.02\linewidth}
\includegraphics[width=0.28\linewidth]{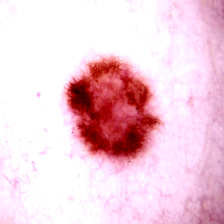}
\hspace{0.04\linewidth}
\includegraphics[width=0.28\linewidth]{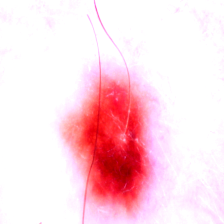}
\hspace{0.04\linewidth}
\includegraphics[width=0.28\linewidth]{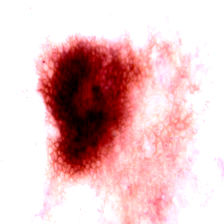}
\includegraphics[width=0.33\linewidth]{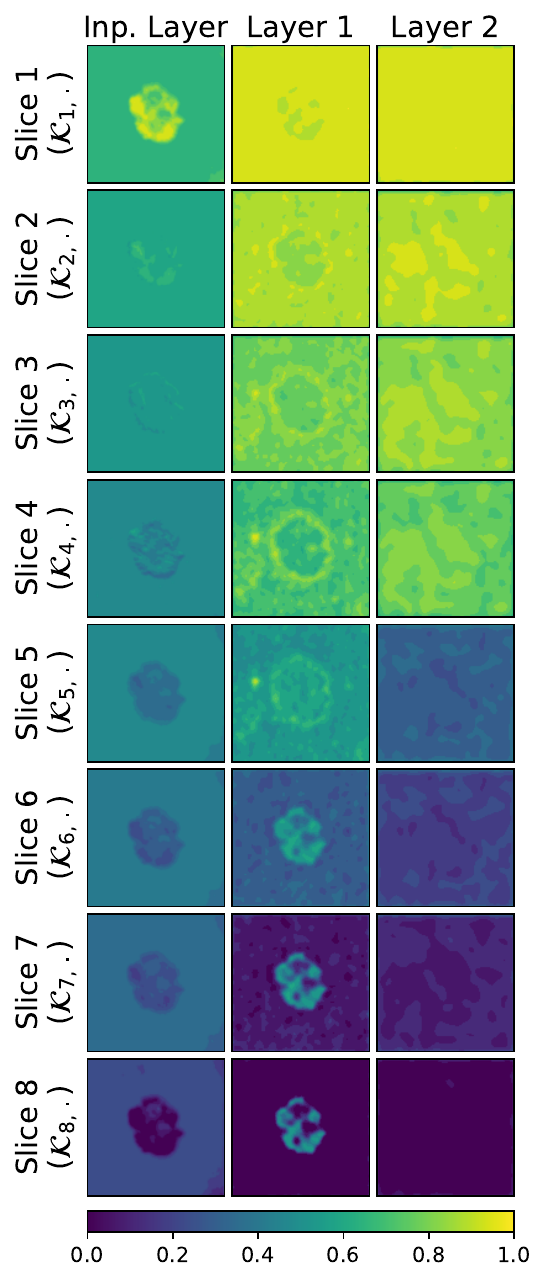}
\includegraphics[width=0.33\linewidth]{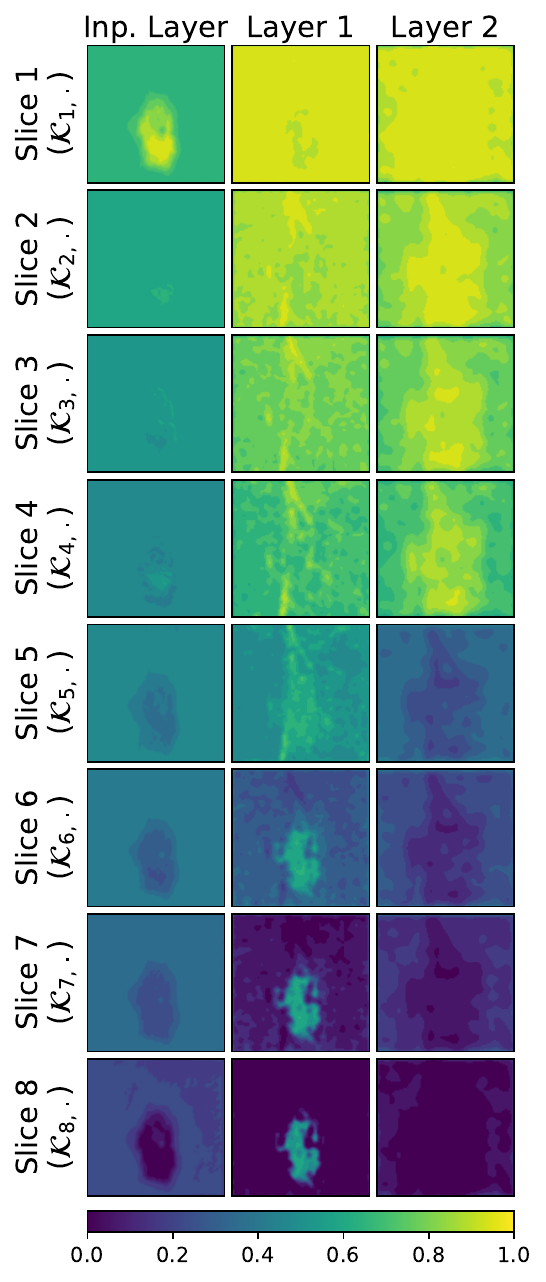}
\includegraphics[width=0.33\linewidth]{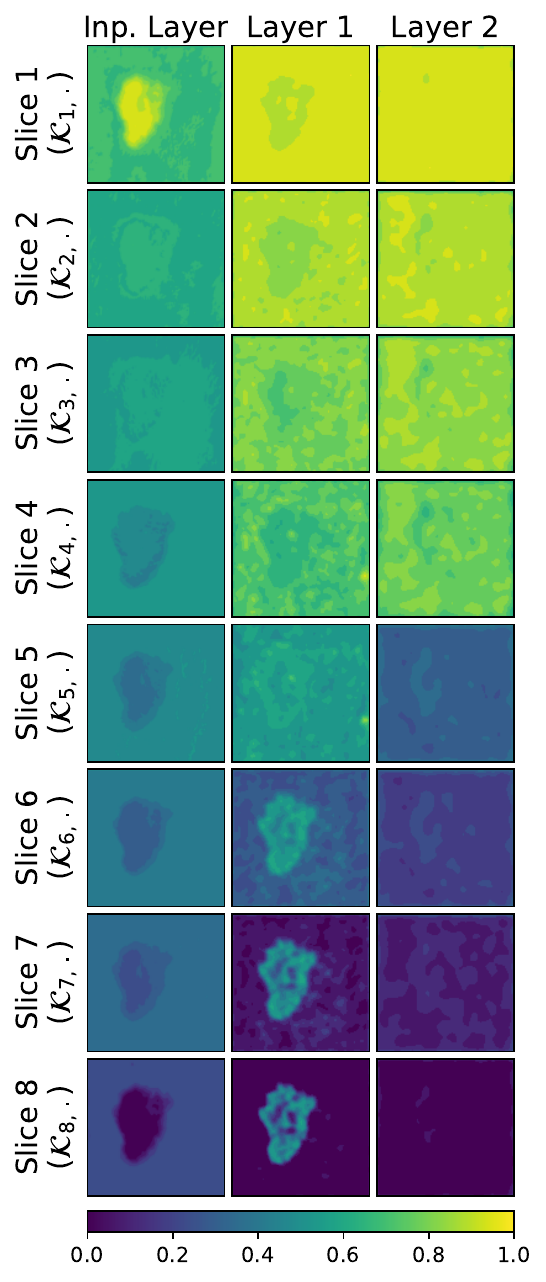}
\caption{
Learned MP Filtration examples for sample images from ISIC dataset~\cite{codella2019skin} for our model TopoSwin-MP with corresponding input images on top. Each slice is a compact representation of one row-wise slice of the learned multifiltration from the input layer, layer 1 and layer 2 learnable cubical multipersistence modules.
\label{fig:isic-learnedmp-compact-filts}}
\end{figure*}
\clearpage

\begin{figure*}[h]
\centering
\hspace{0.03\linewidth}
\includegraphics[width=0.28\linewidth]{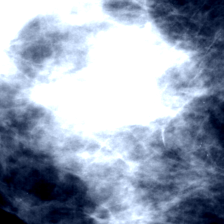}
\hspace{0.045\linewidth}
\includegraphics[width=0.28\linewidth]{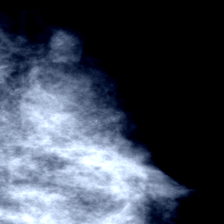}
\hspace{0.038\linewidth}
\includegraphics[width=0.28\linewidth]{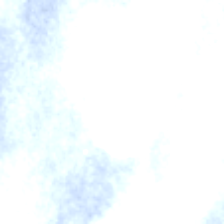}
\includegraphics[width=0.33\linewidth]{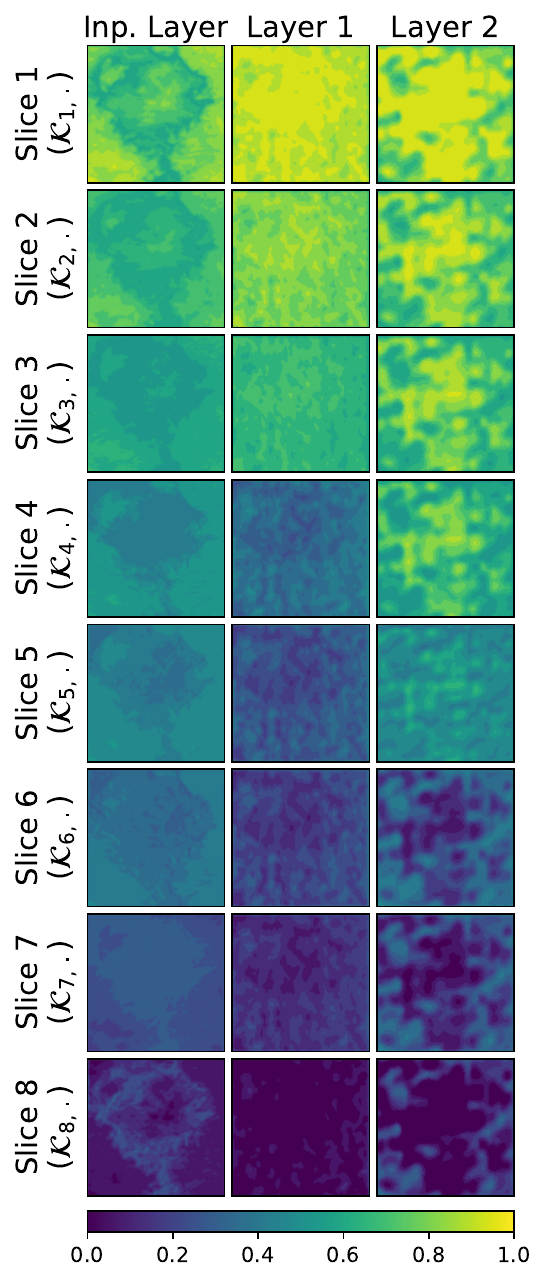}
\includegraphics[width=0.33\linewidth]{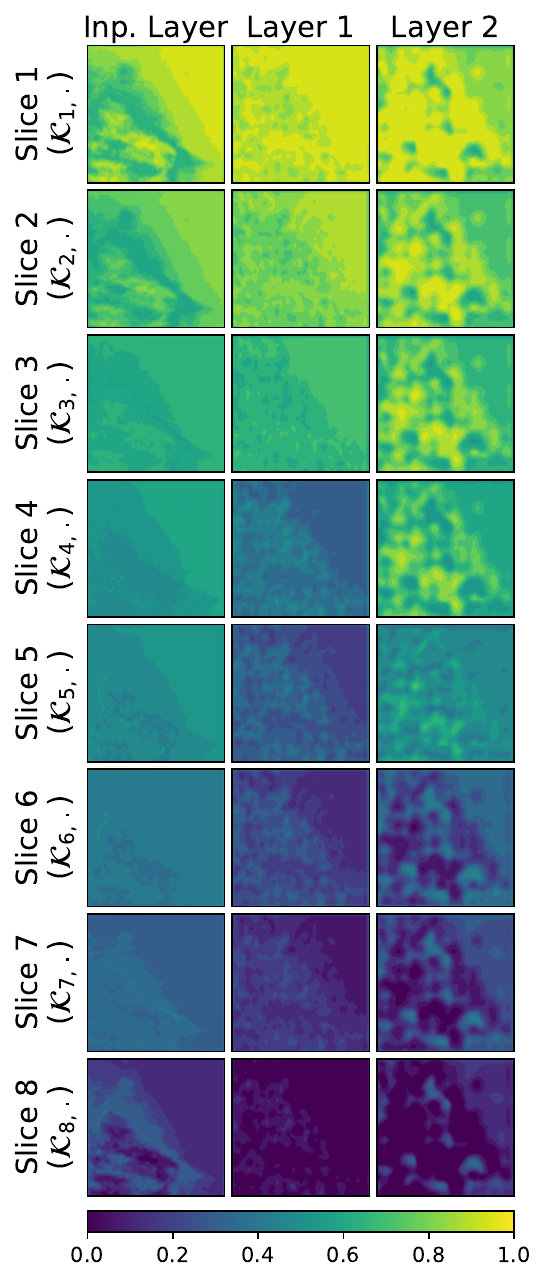}
\includegraphics[width=0.33\linewidth]{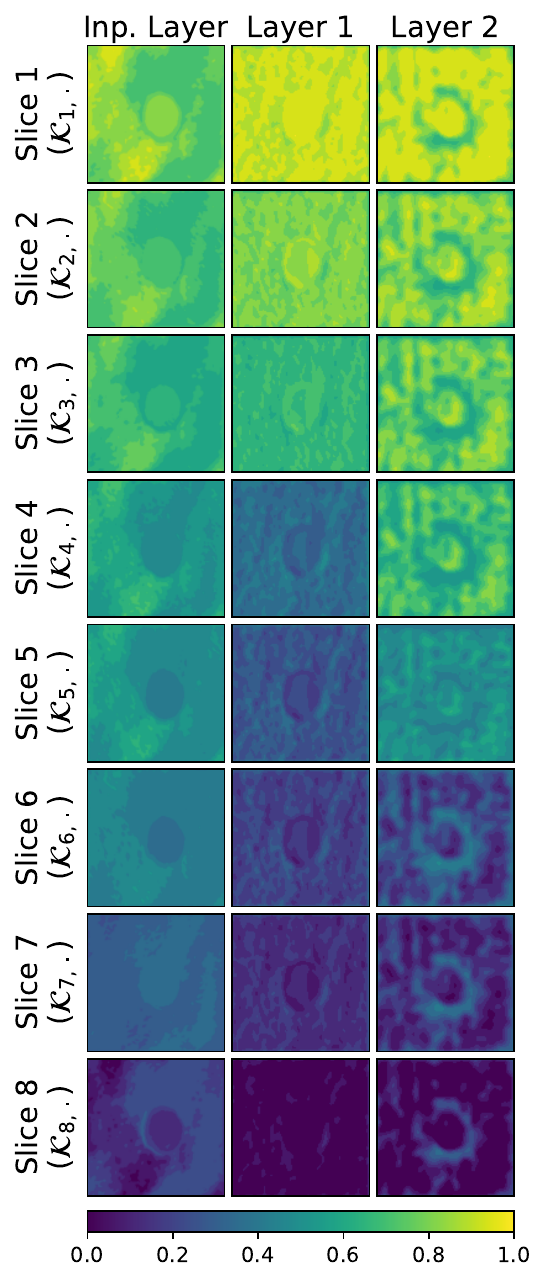}
\caption{
Learned MP Filtration examples for sample images from CBIS-DDSM dataset~\cite{lee2017curated} for our model TopoSwin-MP with corresponding input images on top. Each slice is a compact representation of one row-wise slice of the learned multifiltration from the input layer, layer 1 and layer 2 learnable cubical multipersistence modules.
\label{fig:cbis-learnedmp-compact-filts}}
\end{figure*}
\clearpage

\begin{figure*}[h]
\centering
\hspace{0.013\linewidth}
\includegraphics[width=0.28\linewidth]{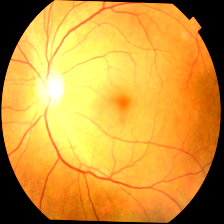}
\hspace{0.045\linewidth}
\includegraphics[width=0.28\linewidth]{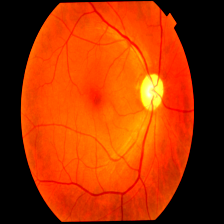}
\hspace{0.045\linewidth}
\includegraphics[width=0.28\linewidth]{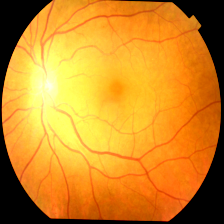}
\includegraphics[width=0.33\linewidth]{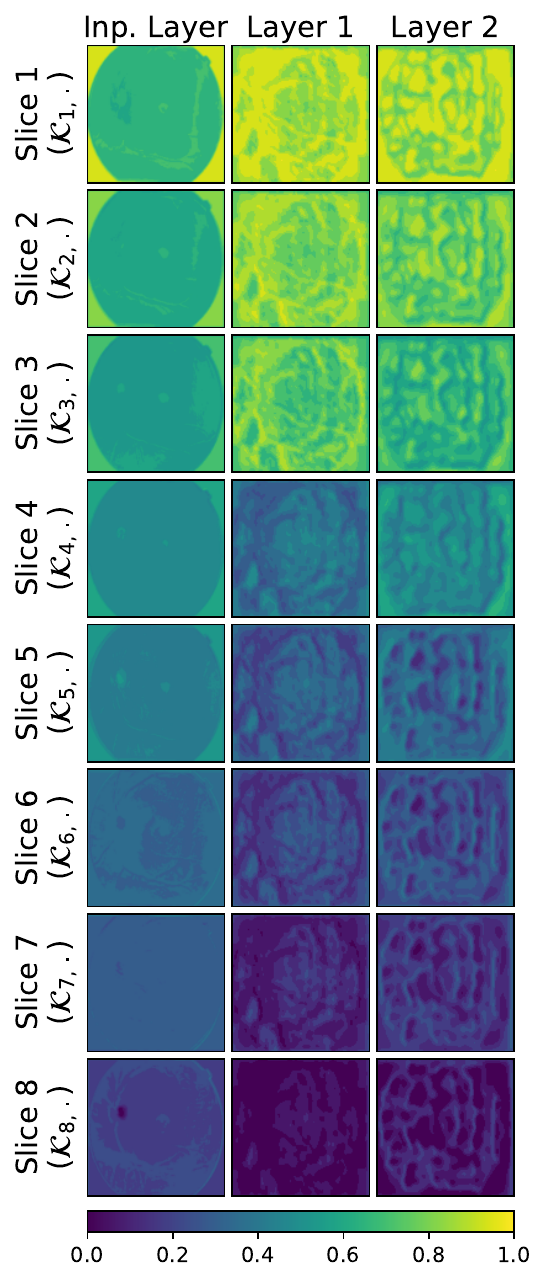}
\includegraphics[width=0.33\linewidth]{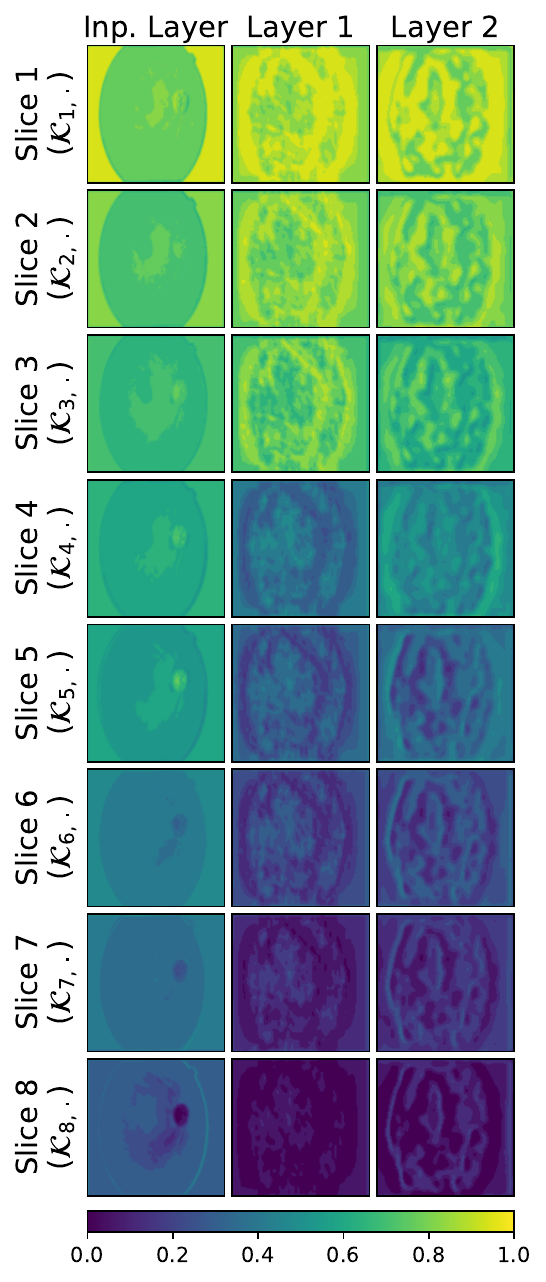}
\includegraphics[width=0.33\linewidth]{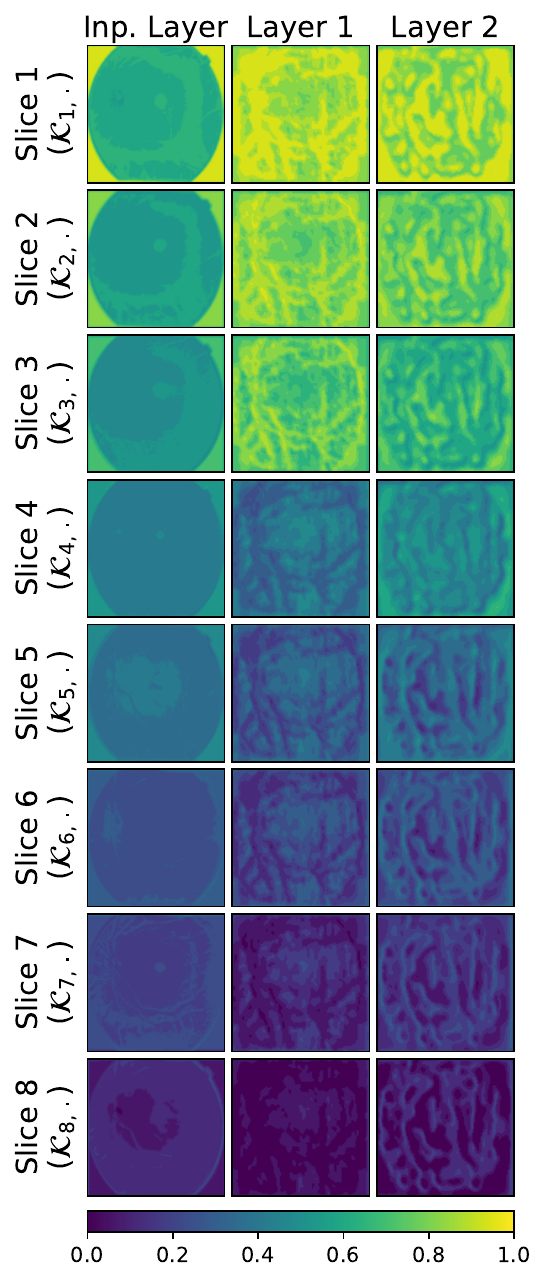}
\caption{
Learned MP Filtration examples for sample images from Glaucoma dataset~\cite{rashid2024eye,sharmin2024dataset} for our model TopoSwin-MP with corresponding input images on top. Each slice is a compact representation of one row-wise slice of the learned multifiltration from the input layer, layer 1 and layer 2 learnable cubical multipersistence modules.
\label{fig:glaucoma-learnedmp-compact-filts}}
\end{figure*}

\clearpage

\begin{figure*}[h]
\centering
\hspace{-0.025\linewidth}
\includegraphics[width=0.27\linewidth]{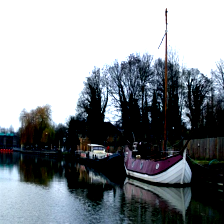}
\hspace{0.032\linewidth}
\includegraphics[width=0.27\linewidth]{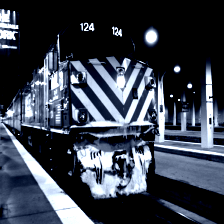}
\hspace{0.035\linewidth}
\includegraphics[width=0.27\linewidth]{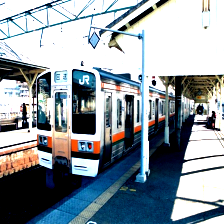}
\includegraphics[width=0.31\linewidth]{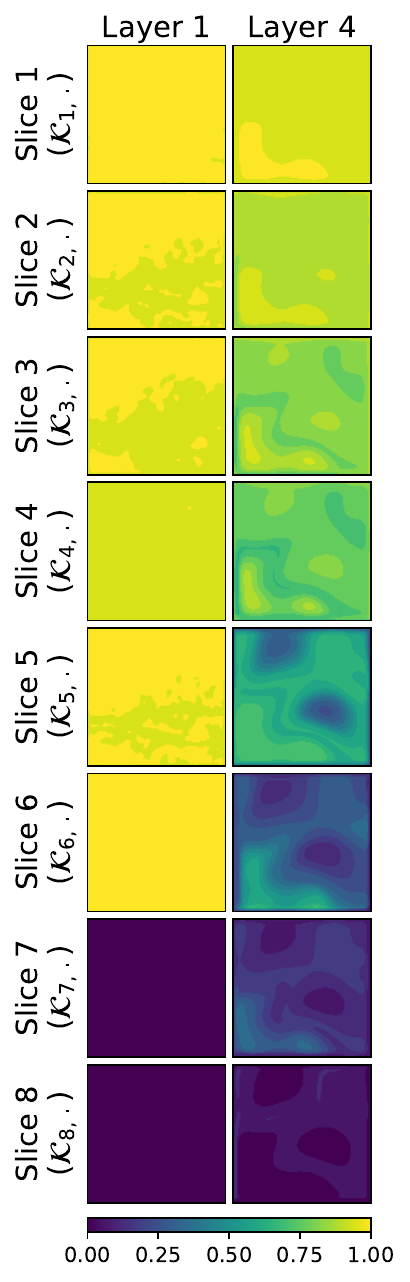}
\includegraphics[width=0.31\linewidth]{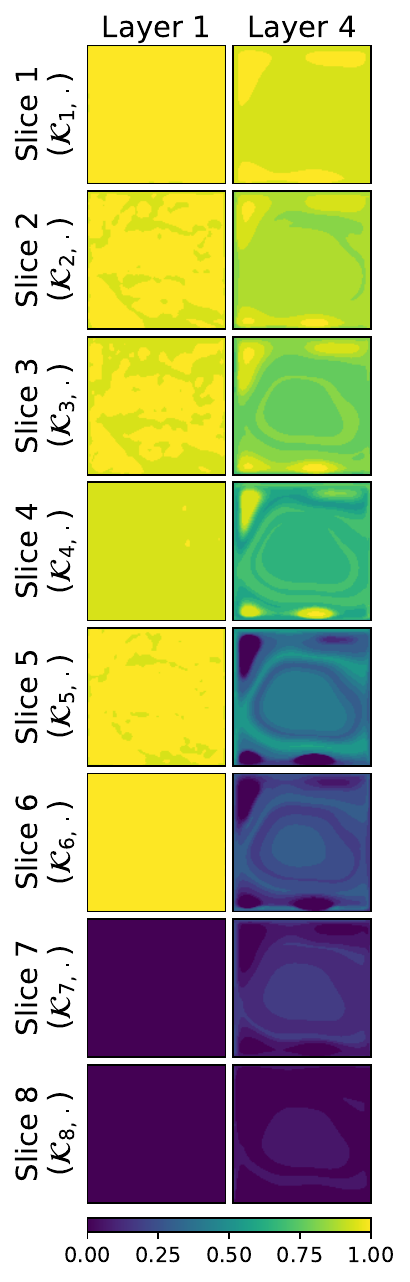}
\includegraphics[width=0.31\linewidth]{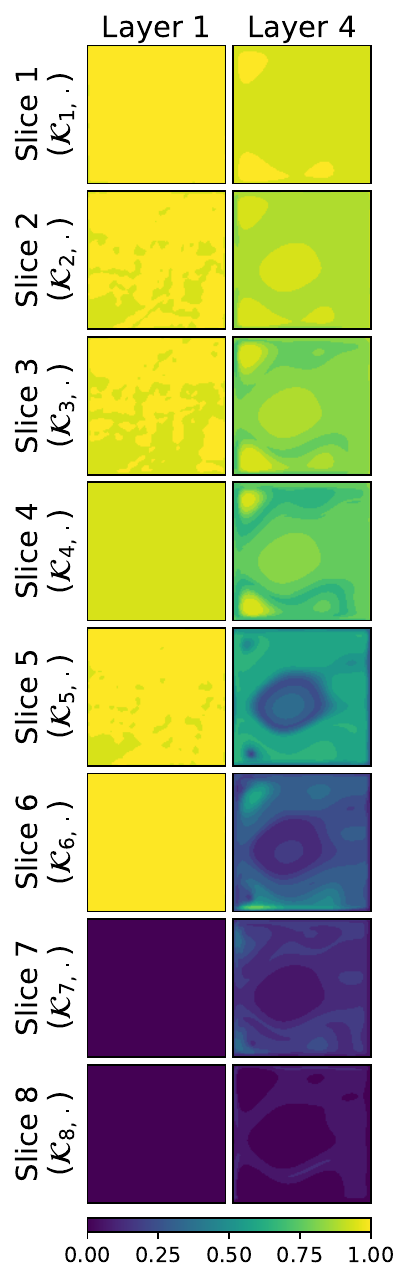}
\caption{
Learned MP Filtration examples for sample images from PASCAL VOC 2012 dataset~\cite{Everingham15} trained with FCN-Resnet 50~\cite{long2015fully} backbone Topo-MP with corresponding input images on top. Each slice is a compact representation of one row-wise slice of the learned multifiltration from the layer 1 and layer 4 learnable cubical multipersistence modules.
\label{fig:voc-resnet-learnedmp-compact-filts}}
\end{figure*}

\clearpage

\begin{figure*}[h]
\centering
\includegraphics[width=0.28\linewidth]{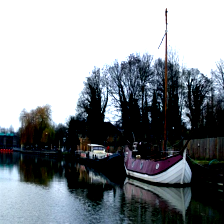}
\hspace{0.024\linewidth}
\includegraphics[width=0.28\linewidth]{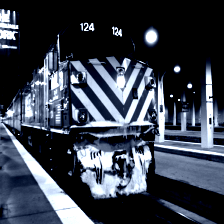}
\hspace{0.024\linewidth}
\includegraphics[width=0.28\linewidth]{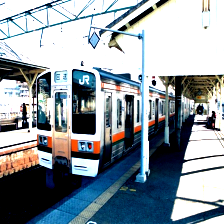}
\hspace{0.01\linewidth}
\includegraphics[width=0.3\linewidth]{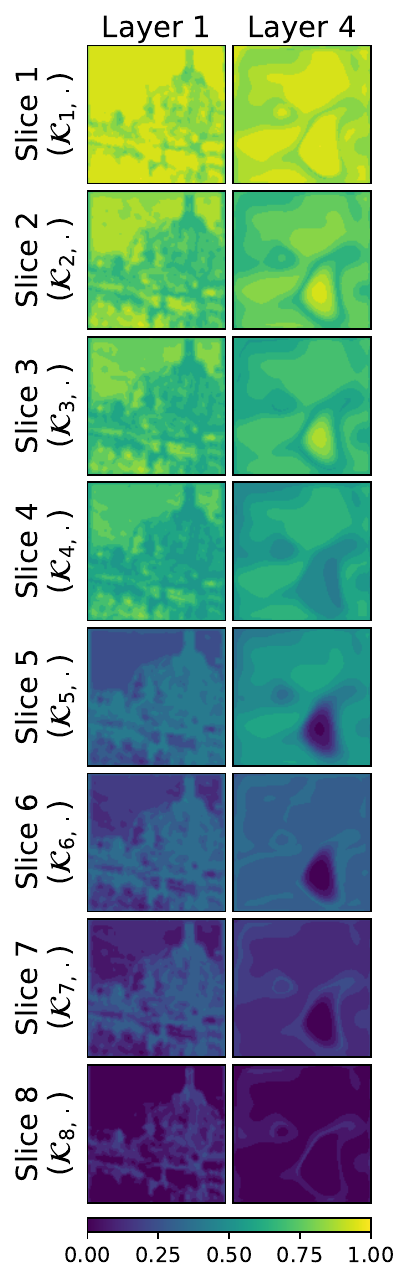}
\hspace{0.01\linewidth}
\includegraphics[width=0.3\linewidth]{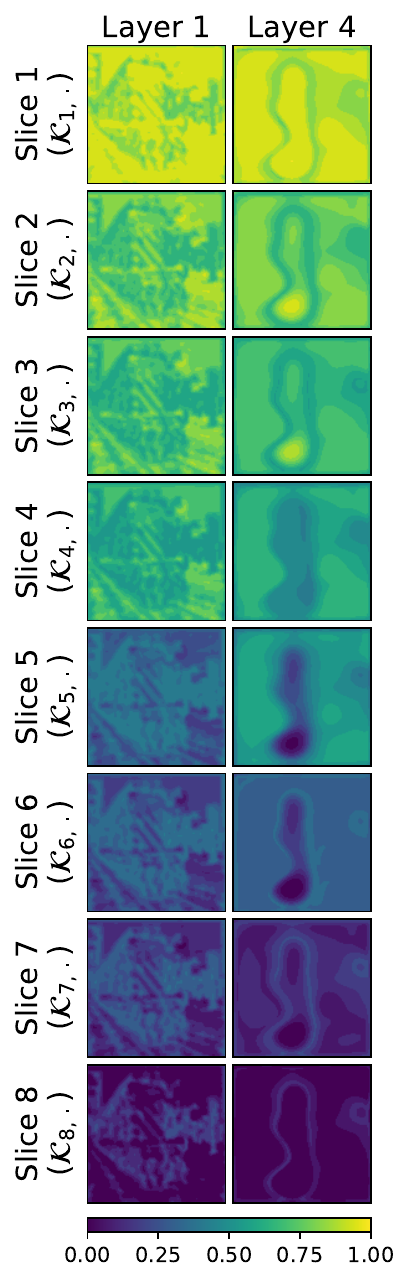}
\hspace{0.01\linewidth}
\includegraphics[width=0.3\linewidth]{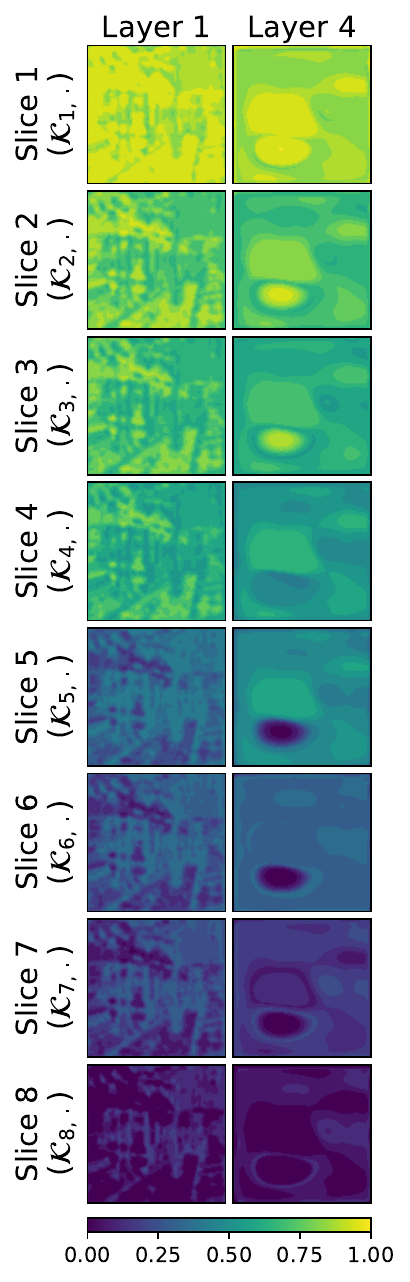}
\caption{
Learned MP Filtration examples for sample images from PASCAL VOC 2012 dataset~\cite{Everingham15} trained with U-Net~\cite{ronneberger2015unet} backbone Topo-MP with corresponding input images on top. Each slice is a compact representation of one row-wise slice of the learned multifiltration from the layer 1 and layer 4 learnable cubical multipersistence modules.
\label{fig:voc-unet-learnedmp-compact-filts}}
\end{figure*}

\clearpage
 
%\else
%\input{ICCV2025/X_suppl}
%\fi

\end{document}